\documentclass{article}

\usepackage[preprint]{neurips_2026}
\usepackage[utf8]{inputenc}
\usepackage[T1]{fontenc}
\usepackage{float}
\usepackage{amsmath,amssymb,amsthm,mathtools}
\usepackage{amsfonts}
\usepackage{bm}

\usepackage{booktabs}
\usepackage{tabularx}
\usepackage{enumitem}
\usepackage{algorithm}
\usepackage[noend]{algpseudocode}
\usepackage{graphicx}
\usepackage{subcaption}

\usepackage{nicefrac}
\usepackage{microtype}
\usepackage{xcolor}
\usepackage{url}
\usepackage{etoolbox}
\definecolor{darkgreen}{RGB}{0,100,0}
\usepackage{wrapfig}
\usepackage{hyperref}
\hypersetup{
    colorlinks=true,
    citecolor=blue,
    linkcolor=black,
    urlcolor=darkgreen
}

\usepackage[nameinlink,capitalize,noabbrev]{cleveref}

\newif\ifcomments
\commentsfalse   

\DeclareMathOperator*{\argmax}{arg\,max}

\newtheorem{assumption}{Assumption}[section]
\newtheorem{lemma}[assumption]{Lemma}
\newtheorem{proposition}[assumption]{Proposition}
\newtheorem{theorem}[assumption]{Theorem}
\newtheorem{corollary}[assumption]{Corollary}
\theoremstyle{remark}
\newtheorem{remark}[assumption]{Remark}

\title{Beyond Static Bias: Adaptive Multi-Fidelity Bandits with Improving Proxies}

\author{
Muyun Lu\textsuperscript{1}\thanks{These authors contributed equally to this work.}
\And
Haoyang Hong\textsuperscript{2}\footnotemark[1]
\And
Huazheng Wang\textsuperscript{2}
\And
Ying Lin\textsuperscript{1}\thanks{Corresponding author: ylin53@Central.UH.EDU}
\\
\textsuperscript{1} Department of Industrial and Systems Engineering, University of Houston \\
\textsuperscript{2} School of  Electrical Engineering and Computer Science, Oregon State University
}
\date{}

\begin{document}

\maketitle

\begin{abstract}
As an extension of the classical multi-armed bandit problem, multi-fidelity multi-armed bandits (MF-MAB) enable individual arms to be evaluated using diverse feedback sources that vary in both cost and accuracy. Prior stochastic models typically assume fixed low-to-high fidelity discrepancies, whereas modern proxy sources, such as learning-based simulators and Large Language Models (LLMs), can be improved using additional calibration. We investigate adaptive MF-MAB with improving proxy sources, and focus on the canonical two-fidelity case in which the low-fidelity source becomes more informative with repeated use. To capture this dynamic, we introduce a selected-average mismatch bound that converts dynamic low-fidelity observations into improvement-aware confidence bounds for the high-fidelity target. We propose the Threshold-Based Adaptive Continuation Companion (TACC), an optimistic algorithm that uses a bounded continuation rule to decide when low-fidelity sampling remains cost-effective and when to escalate. We prove an instance-dependent regret bound showing that, for detected intermediate arms, adaptive continuation replaces logarithmic high-fidelity confirmation with bounded low-fidelity continuation. Experiments on synthetic bandits and an LLM-as-a-judge policy-evaluation task examine when continuation improves cost-weighted regret.
\end{abstract}

\section{Introduction}

While multi-armed bandit (MAB) is a standard model for sequential decision-making \citep{lai1985asymptotically, auer2002finite, villar2015clinical, li2010contextual, bouneffouf2019survey}, the feedback one ultimately cares about is often expensive to obtain. Multi-fidelity bandits (MF-MAB) make this trade-off explicit: in
online advertising, for example, one may first test an ad on a shorter time window
or a narrower audience, using this biased but cheaper signal to discard clearly
inferior candidates before paying for full-scale evaluation
\citep{kandasamy2016multi}. This principle uses low fidelities to eliminate bad arms and reserves high-fidelity pulls for
arms that remain close to optimal; it has also been extended to fixed-confidence
best-arm identification and regret minimization with fidelity-dependent costs and
accuracies \citep{wang2023revisited, poiani2022mfbaibandit,poiani2024optimal}.


Traditional MF-MAB frameworks typically assume a static bias between low- and high-fidelity sources. In modern applications, however, a proxy source may improve as additional budget is spent on it. Learning-based simulators can be refined with additional calibration budget, and LLM-based evaluators can be calibrated or fine-tuned with additional feedback \citep{ouyang2022training,bai2022constitutional,yuan2024self}. When such systems serve as low-fidelity sources, their discrepancy from the target evaluator is better viewed as dynamic rather than fixed. The low-fidelity budget then has two effects: it provides immediate observations and can make future low-fidelity information more reliable.



We investigate multi-fidelity multi-armed bandits in which proxy sources may improve through use. As a foundational setting, we focus on a two-fidelity bandit problem with cost-weighted regret, where a stationary high-fidelity source defines the target means and a low-fidelity source whose discrepancy depends on the number of low-fidelity queries. Rather than assuming a fixed bound of bias between fidelities, we summarize the evolving discrepancy through a selected-average mismatch bound. This bound yields target-calibrated, improvement-aware confidence bounds for the high-fidelity target, allowing the low-fidelity proxy to become increasingly informative as it is queried.

The question is not simply whether low fidelity is accurate enough now, but whether it is still worth improving before paying for high fidelity. A static threshold rule escalates as soon as the low-fidelity statistical uncertainty falls below a prescribed level. This is appropriate for arms whose low-fidelity proxy already certifies suboptimality, and also for arms whose proxy remains too misaligned to be useful. The ambiguous case lies in between: after the statistical radius has crossed the threshold, a few additional low-fidelity queries may still reduce the mismatch correction enough to certify the arm without high-fidelity confirmation. Threshold-Based Adaptive Continuation Companion (TACC) is designed to exploit this intermediate regime. It combines optimistic arm selection with a bounded post-threshold continuation rule, which decides whether additional low-fidelity samples remain cost-effective based on the anticipated decrease in the mismatch correction. Our analysis formalizes this intuition through an instance-dependent regret bound, showing that detected intermediate arms can be certified through bounded low-fidelity continuation rather than through logarithmic high-fidelity confirmation.

\paragraph{Contributions.}
Our contributions are as follows.
\begin{enumerate}[leftmargin=1.5em]
    \item In Section~\ref{sec:problem}, we investigate MF-MAB with improving proxy feedback and formulate a canonical two-fidelity model where the low-fidelity source improves with use. For this setting, we introduce a selected-average mismatch bound that yields improvement-aware confidence bounds for the high-fidelity target.

    \item In Section~\ref{sec:method}, we propose the TACC algorithm, pairing optimistic arm selection with a continuation rule to efficiently time escalations to high-fidelity queries.

    \item In Section~\ref{sec:theory}, we derive an instance-dependent regret bound that isolates the role of adaptive continuation: for the detected intermediate arms, the bound replaces a logarithmic high-fidelity confirmation term with a bounded low-fidelity continuation cost.

    \item In Section~\ref{sec:exp}, we empirically examine when the continuation mechanism improves cost-weighted regret in synthetic simulations and a controlled LLM-as-a-judge policy-evaluation experiment.
\end{enumerate}


\section{Related Work}

\paragraph{Multi-fidelity and structured non-stationary bandits.}
Our work builds on stochastic multi-armed bandits and their multi-fidelity extensions \citep{lattimore2020bandit,kandasamy2016multi,poiani2022mfbaibandit,wang2023revisited, poiani2024optimal}. Standard multi-fidelity methods manage the cost--accuracy trade-off by combining statistical uncertainty with a fixed or uniformly bounded fidelity bias; this viewpoint also underlies multi-fidelity optimization and hyperparameter tuning methods such as multi-fidelity Bayesian optimization and Hyperband \citep{kandasamy2019multi,li2018hyperband}. In contrast, our setting allows the low-fidelity discrepancy to evolve with the number of low-fidelity queries. This connects our model to structured non-stationary bandits, including improving/decaying bandits, rotting bandits, and rising bandits \citep{heidari2016tight,seznec2019rotting,metelli2022stochastic}. The main difference is that we keep the high-fidelity target stationary while letting only the low-fidelity proxy improve, so the fidelity choice itself becomes history-dependent.


\paragraph{Bandits with evolving information sources.}
Departing from the traditional assumption of statistically stationary arms, a parallel body of work investigates settings where rewards evolve dynamically \citep{tekin2012online,komiyama2014time,cheung2019learning}. Within this space, structured variants capture pull-dependent trajectories, such as performance degradation in rotting bandits \citep{levine2017rotting,seznec2019rotting}, enhancement in rising bandits \citep{mussi2023best}, and coupled cross-arm dependencies in graph-triggered environments \citep{genalti2024graph}. Closely related are adaptive model-selection strategies that dynamically choose among candidate algorithms or model classes based on bandit feedback \citep{foster2017parameter,agarwal2017corralling,pacchiano2020model,cutkosky2021dynamic,dann2024data}. Such evolving-source dynamics also emerge when arms represent independent learners that improve with increased budget \citep{cella2021best,balef2025put,latypov2025ucb}, as well as in functional bandits characterized by complex feedback mechanisms \citep{tran2014functional,dorn2025functional}. Adopting the assumptions from this literature, our work introduces the evolving-source perspective to the multi-fidelity setting: we target a stationary high-fidelity mean by leveraging an improving low-fidelity proxy, whose approximation error diminishes with repeated use.
\section{Problem Setup}
\label{sec:problem}


We study a budget-constrained stochastic multi-armed bandit problem featuring a cheaper, improving low-fidelity source ($L$) and a target, high-fidelity source ($H$). 

There are $K$ arms indexed by $k\in[K]$. The target value of arm $k$ is its high-fidelity mean $\mu_k^{(H)}$. Let $k^\star\in\argmax_{k\in[K]}\mu_k^{(H)}$ represent the optimal arm, $\mu^\star:=\mu_{k^\star}^{(H)}$ denote its target value, and $\Delta_k:=\mu^\star-\mu_k^{(H)}$ for $k\neq k^\star$ define the suboptimality gap. Querying each fidelity incurs associated cost, denoted by $\lambda^{(L)}$, $\lambda^{(H)}$, with $0<\lambda^{(L)}<\lambda^{(H)}$. Given a total budget $\Lambda>0$, define $T_\Lambda:=\left\lceil \Lambda/\lambda^{(L)}\right\rceil$, which upper bounds the total number of queries under budget $\Lambda$.

At round $t$, the learner selects an arm-fidelity pair $(I_t,m_t)\in[K]\times\{L,H\}$, observes a noisy sample, and pays cost $\lambda^{(m_t)}$. Let $C_t:=\sum_{s=1}^t \lambda^{(m_s)}$ be the cumulative cost after $t$ rounds. We measure performance of an algorithm by the cost-weighted pseudo-regret
\begin{equation}
R(\Lambda)
:=
\sum_{t:C_t\le \Lambda}
\lambda^{(m_t)}
\bigl(\mu^\star-\mu_{I_t}^{(H)}\bigr).
\end{equation}
This formulation follows standard setup in multi-fidelity bandits \citep{kandasamy2016multi}, although \citet{wang2023revisited} study a different regret definition. Formally, querying arm $k$ at high fidelity for the $u$-th time yields $Y_{k,u}^{(H)} = \mu_k^{(H)} + \eta_{k,u}^{(H)}$, whereas the $\tau$-th low-fidelity query yields $Y_{k,\tau}^{(L)} = \mu_k^{(L)}(\tau) + \eta_{k,\tau}^{(L)}$. The noise components, $\eta_{k,u}^{(H)}$ and $\eta_{k,\tau}^{(L)}$, are zero-mean and $1$-sub-Gaussian.

The key modeling choice is that the low-fidelity evolves according to arm-specific query counts. Analogous to rested bandit models \citep{cella2021best}, the low-fidelity response of arm $k$ updates solely upon being queried. We deliberately exclude global-time drift and cross-arm coupling to clearly isolate how localized fidelity improvements influence fidelity choice.

For $n\ge 1$, the selected-average low-fidelity mean is defined as $\bar{\mu}_k^{(L)}(n):=\frac{1}{n}\sum_{\tau=1}^{n}\mu_k^{(L)}(\tau)$. Under the count-indexed model, this is the quantity estimated by the low-fidelity empirical mean. To capture the discrepancy between this evolving low-fidelity source and the high-fidelity target, we use a cumulative mismatch certificate, building upon recent certificate-based analyses for improving information sources \citep{latypov2025ucb}.


\begin{assumption}[Evolving low-fidelity mismatch certificate]
\label{ass:plugin}
For each arm $k$, there exists a known nondecreasing function $U_k(\cdot):\mathbb{N}\to\mathbb{R}_+$ such that, for every $n\ge 1$,
\begin{equation}
\sum_{\tau=1}^{n}
\left|
\mu_k^{(L)}(\tau)-\mu_k^{(H)}
\right|
\le
U_k(n).
\end{equation}
Define the associated selected-average mismatch bound by $B_k(n):=\sup_{s\ge n} U_k(s)/s$ for $n\ge 1$. Then
\begin{equation}
\label{eq:improvingbound}
\left|
\bar{\mu}_k^{(L)}(n)-\mu_k^{(H)}
\right|
\le
\frac{U_k(n)}{n}
\le
B_k(n).
\end{equation}
\end{assumption}




\begin{remark}[Relation to static multi-fidelity models]
\label{rem:evolving_vs_static}
Standard multi-fidelity bandits assume a fixed bias bound, such as $|\mu_k^{(L)}-\mu_k^{(H)}|\le \zeta_k$ \citep{kandasamy2016multi,wang2023revisited}. Assumption~\ref{ass:plugin} replaces this by an evolving mismatch certificate. The induced selected-average mismatch bound $B_k(n)$ controls
\(
|\bar{\mu}^{(L)}_k(n)-\mu^{(H)}_k|,
\)
which is the discrepancy relevant to the low-fidelity empirical mean. Inside the confidence interval, $B_k(n)$ serves as a mismatch correction that relates the selected-average low-fidelity mean to the high-fidelity target.

\end{remark}

Given $B_k(n)$, the effective low-fidelity gap for each suboptimal arm $k\neq k^\star$ can be defined by
\begin{equation}
\Delta_k^{(L)}(n)
:=
\mu^\star-\bar{\mu}_k^{(L)}(n)-B_k(n).
\label{eq:effective-low-gap}
\end{equation}
When $\Delta_k^{(L)}(n)>0$, the low-fidelity upper proxy $\bar{\mu}_k^{(L)}(n)+B_k(n)$ lies below the optimal high-fidelity value $\mu^\star$, so arm $k$ can in principle be ruled out using low-fidelity information. Using the mismatch correction $B_k(n)$, the algorithm can form target-calibrated, improvement-aware upper confidence bounds for the high-fidelity target. Section~\ref{sec:theory} then uses $\Delta_k^{(L)}(n)$ to formalize when low-fidelity continuation certifies suboptimality.

In the next section, we introduce Threshold-Based Adaptive Continuation Companion (TACC), a confidence-based policy that uses this structure to decide whether low-fidelity sampling remains worthwhile before switching to high fidelity.
\section{Algorithm}
\label{sec:method}

Our algorithm follows the standard confidence-based template: maintain upper bounds for the target mean, select an optimistic arm, and decide how that arm should be sampled. The additional issue in our setting is the post-threshold fidelity choice. Once the low-fidelity uncertainty has fallen below the usual switching threshold, a static rule would escalate to high fidelity; an improving low-fidelity source, however, may still become more useful after a few cheap queries. Threshold-Based Adaptive Continuation Companion (TACC) captures this possibility through a bounded continuation rule that uses the anticipated decrease in the mismatch correction to decide whether additional low-fidelity sampling is still worthwhile.


\subsection{Confidence bounds}
\label{subsec:confidence-bounds}
We use both fidelities to build confidence intervals for the same object: the high-fidelity target mean $\mu_k^{(H)}$. For high fidelity, this is a standard statistical interval around the empirical high-fidelity mean. For low fidelity, the empirical mean estimates the selected-average low-fidelity mean, so we enlarge the interval by the mismatch correction $B_k(n)$ before using it as a valid interval for $\mu_k^{(H)}$.

For the budget horizon $\Lambda$, we use the budget-uniform confidence radius
\[G_\Lambda(n,\delta):=\sqrt{\frac{\rho \log(2KT_\Lambda/\delta)}{n}},\qquad n\ge 1,\]
where $\rho>0$ is a problem-dependent constant. This radius controls the sampling error of an empirical mean after $n$ observations, uniformly over arms, fidelities, and all query counts up to the budget horizon \citep{auer2002using, audibert2009exploration}. Let $N_{k,t}^{(L)}$ and $N_{k,t}^{(H)}$ denote the numbers of low- and high-fidelity queries of arm $k$ before round $t$, and let $\hat{\mu}_{k,t}^{(L)}$ and $\hat{\mu}_{k,t}^{(H)}$ denote the corresponding empirical means. The confidence radii are defined as:
\[
\operatorname{rad}_{k}^{(L)}(n,\delta):=G_\Lambda(n,\delta)+B_k(n),
\quad
\operatorname{rad}^{(H)}(n,\delta):=G_\Lambda(n,\delta).
\]


The corresponding confidence bounds from fidelity $m\in\{L,H\}$ for the target
mean $\mu_k^{(H)}$ are\footnote{For high fidelity, $\operatorname{rad}_{k}^{(H)}(n,\delta)$  is identical to $\text{rad}^{(H)}$ here.}
\begin{equation}
\begin{aligned}
\mathrm{UCB}_{k,t}^{(m)}
&=
\hat{\mu}_{k,t}^{(m)}
+
\operatorname{rad}_{k}^{(m)}(N_{k,t}^{(m)},\delta),\\
\mathrm{LCB}_{k,t}^{(m)}
&=
\hat{\mu}_{k,t}^{(m)}
-
\operatorname{rad}_{k}^{(m)}(N_{k,t}^{(m)},\delta).
\end{aligned}
\label{eq:fidelity-confidence-bounds}
\end{equation}
Both fidelities therefore yield valid intervals for the same target quantity $\mu_k^{(H)}$, but with different cost--accuracy trade-offs: low fidelity carries the mismatch correction $B_k(n)$, whereas high fidelity does not. This construction parallels standard multi-fidelity confidence designs, where fidelity-specific intervals are combined to estimate the same target mean, but here the low-fidelity mismatch is count-dependent through $B_k(n)$ rather than a fixed fidelity bias bound \citep{wang2023revisited}.

\subsection{Arm and fidelity selection}
\label{subsec:arm-and-fidelity-selection}

The decision rule has two components: an arm-selection step based on the aggregate confidence score, followed by a fidelity-selection step for the chosen arm.

\paragraph{Arm selection.}
Once the two target intervals are available, we choose the arm that is still most plausible under the high-fidelity objective. Since either fidelity can already rule out an arm as suboptimal for the target mean, we define the aggregate optimistic score as the smaller of the two fidelity-specific upper bounds:
\begin{equation}
\mathrm{UCB}_{k,t}
:=
\min_{m\in\{L,H\}}\mathrm{UCB}_{k,t}^{(m)},
\qquad
I_t\in\argmax_{k\in[K]}\mathrm{UCB}_{k,t}.
\label{eq:arm-score}
\end{equation}
Thus an arm remains competitive only if neither fidelity-specific upper bound has ruled it out.

\paragraph{Fidelity selection with bounded continuation.}
After selecting $I_t$, we decide which fidelity to query next for that arm. If
\(G_\Lambda(N_{I_t,t}^{(L)},\delta)\ge \gamma\),
the low-fidelity estimate is still statistically too coarse, so we continue querying low fidelity.

The more delicate case occurs after this radius falls below $\gamma$. A static threshold rule would now escalate to high fidelity. In contrast, we ask whether a few additional low-fidelity queries could still reduce the mismatch correction enough to be worthwhile. For the selected arm $I_t$, define the continuation gain
\begin{equation}
V_{I_t,t}(s):=B_{I_t}\!\left(N_{I_t,t}^{(L)}\right)-B_{I_t}\!\left(N_{I_t,t}^{(L)}+s\right),
\qquad 1\le s\le S_0.
\label{eq:continuation-gain}
\end{equation}
This quantity is deterministic given the current low-fidelity count and measures how much the mismatch correction would decrease after $s$ further low-fidelity queries.

Let $A_k$ denote the number of continuation pulls already spent on arm $k$ after its low-fidelity radius first drops below $\gamma$. The fidelity-selection rule is
\begin{equation}
m_t=
\begin{cases}
L,
&\text{if }
G_\Lambda\!\left(N_{I_t,t}^{(L)},\delta\right)\ge \gamma
\ \text{or}\
\Big[
A_{I_t}<S_0,\ 
\exists\,1\le s\le S_0-A_{I_t}: V_{I_t,t}(s)\ge 2\eta\gamma
\Big],
\\[2pt]
H,
&\text{otherwise}.
\end{cases}
\label{eq:fidelity-selection}
\end{equation}
Here, $\eta\in(0,1)$ controls how readily we continue low-fidelity sampling after the threshold crossing. We impose \(S_0\lambda^{(L)}\le \lambda^{(H)}\), so a full continuation block costs no more than one high-fidelity query.


\begin{algorithm}[H]
\caption{Threshold-Based Adaptive Continuation Companion (TACC)}
\label{alg:adaptive-mf-mab}
\begin{algorithmic}[1]
\State \textbf{Input:} budget $\Lambda$, arm set $[K]$, costs $\lambda^{(L)},\lambda^{(H)}$, confidence level $\delta$, threshold $\gamma$, continuation parameter $\eta$, continuation cap $S_0$
\State Query each arm once at each fidelity; initialize $\hat{\mu}_{k}^{(L)},\hat{\mu}_{k}^{(H)},N_{k,t}^{(L)},N_{k,t}^{(H)}$ and $A_k\leftarrow 0$
\State Set cumulative cost $C\leftarrow K(\lambda^{(L)}+\lambda^{(H)})$
\While{$C\le \Lambda$}
    \State Construct $\mathrm{UCB}_{k,t}^{(L)},\mathrm{LCB}_{k,t}^{(L)},\mathrm{UCB}_{k,t}^{(H)},\mathrm{LCB}_{k,t}^{(H)}$ for all $k\in[K]$ with \eqref{eq:fidelity-confidence-bounds}
    \State Compute $\mathrm{UCB}_{k,t}$ with \eqref{eq:arm-score} and select $I_t\in\argmax_{k\in[K]}\mathrm{UCB}_{k,t}$
    \If{$G_\Lambda(N_{I_t,t}^{(L)},\delta)\ge \gamma$}
        \State Set $m_t=L$
    \ElsIf{$A_{I_t}<S_0$ and there exists $1\le s\le S_0-A_{I_t}$ such that $V_{I_t,t}(s)\ge 2\eta\gamma$}
        \State Set $m_t=L$ and update $A_{I_t}\leftarrow A_{I_t}+1$
    \Else
        \State Set $m_t=H$
    \EndIf
    \If{$C+\lambda^{(m_t)}>\Lambda$}
        \State \textbf{break}
    \EndIf
    \State Query arm $I_t$ at fidelity $m_t$, observe reward, and update empirical means and query counts
    \State Update cumulative cost $C\leftarrow C+\lambda^{(m_t)}$
\EndWhile
\end{algorithmic}
\end{algorithm}

Algorithm~\ref{alg:adaptive-mf-mab} initializes with one query per arm at each fidelity. In each later round, it builds improvement-aware confidence intervals, selects the arm with the largest optimistic index, and chooses the fidelity. It samples low fidelity while its radius exceeds $\gamma$; after crossing $\gamma$, it continues low fidelity only if the mismatch-decrease test warrants a bounded continuation block, and otherwise switches to high fidelity. A pre-query check enforces the budget $\Lambda$.


\section{Theoretical Result}
\label{sec:theory}

We analyze Algorithm~\ref{alg:adaptive-mf-mab} on the budget-uniform event
$\mathcal{E}_\Lambda$ from Lemma~\ref{lem:ci}, which holds with probability
at least $1-\delta$. Let
\(
\ell_\Lambda:=\rho\log\!\left(\frac{2KT_\Lambda}{\delta}\right),
\,\,
G_\Lambda(n,\delta):=\sqrt{\frac{\ell_\Lambda}{n}},
\,\,
N_\gamma:=\min\{n\ge 1:G_\Lambda(n,\delta)<\gamma\}.
\)
The logarithmic factor is chosen only to support a union bound over arms,
fidelities, and sample counts up to the budget horizon.

\subsection{Main result}
\label{subsec:analysis-main-result}

For each suboptimal arm $k\neq k^\star$, define the low-fidelity certification
time
\begin{equation}
\tau_k(\gamma)
:=
\inf\left\{
n\ge 1:
\Delta_k^{(L)}(n)\ge 2\gamma
\right\},
\label{eq:analysis-cert-time}
\end{equation}
with $\tau_k(\gamma)=\infty$ if the set is empty. Thus $\tau_k(\gamma)$ is
the first low-fidelity sample count at which the deterministic low-fidelity
certificate for arm $k$ reaches level $2\gamma$.

The suboptimal arms are partitioned into
\begin{equation}
\begin{aligned}
\mathcal{A}_\gamma
&:=
\left\{
k\neq k^\star:
\tau_k(\gamma)\le N_\gamma
\right\},\\
\mathcal{B}_\gamma
&:=
\left\{
k\neq k^\star:
\tau_k(\gamma)>N_\gamma+S_0
\right\},\\
\mathcal{C}_\gamma
&:=
\left\{
k\neq k^\star:
N_\gamma<\tau_k(\gamma)\le N_\gamma+S_0
\right\}.
\end{aligned}
\label{eq:analysis-arm-partition}
\end{equation}
Arms in $\mathcal{A}_\gamma$ are low-fidelity certifiable by the nominal
switching threshold. Arms in $\mathcal{B}_\gamma$ are not certifiable within
the allowed continuation window. The intermediate class $\mathcal{C}_\gamma$
contains arms that are not certifiable at the static threshold, but can become
certifiable after at most $S_0$ additional low-fidelity queries.

The refinement of $\mathcal{C}_\gamma$ is execution dependent. Let
$\mathcal{C}_\gamma^{\mathrm d}$ denote the arms in $\mathcal{C}_\gamma$ whose
realized trajectory reaches the certification count $\tau_k(\gamma)$ within the
budget horizon, and for which every selected round satisfying
$N_\gamma\le N_{k,t}^{(L)}<\tau_k(\gamma)$ is queried at low fidelity. Let
$\mathcal{C}_\gamma^{\mathrm u}:=\mathcal{C}_\gamma\setminus
\mathcal{C}_\gamma^{\mathrm d}$ and
$\mathcal{H}_\gamma:=\mathcal{B}_\gamma\cup
\mathcal{C}_\gamma^{\mathrm u}$. A formal pathwise definition of
$\mathcal{C}_\gamma^{\mathrm d}$ is given in
\eqref{eq:app-analysis-Cdet}. Thus $\mathcal{H}_\gamma$ collects the arms for
which the analysis retains the logarithmic high-fidelity confirmation term.

The regret analysis uses one additional regularity condition, which ensures
that a low-fidelity certificate, once reached within the continuation window,
does not disappear before the window ends.

\begin{assumption}[Local Low-fidelity certificate stability]
\label{ass:mono-cert}
Fix the threshold $\gamma>0$ and continuation cap $S_0$ used in Algorithm~\ref{alg:adaptive-mf-mab}. Let $N_\gamma:=\min\{n\ge 1: G_\Lambda(n,\delta)<\gamma\}$ denote the nominal low-fidelity threshold associated with $\gamma$, where $G_\Lambda$ is the budget-uniform statistical radius defined in Section~\ref{sec:method}. For each suboptimal arm $k\neq k^\star$, if $\Delta_k^{(L)}(n)\ge 2\gamma$ for some $n\le N_\gamma+S_0$, then $\Delta_k^{(L)}(m)\ge 2\gamma$ for all $m\in[n,\,N_\gamma+S_0]$.
\end{assumption}

\begin{remark}
Assumption~\ref{ass:mono-cert} is only imposed on the finite window relevant
to the continuation rule. It requires that, once the low-fidelity certificate
reaches level $2\gamma$ before the continuation cap is exhausted, the same
certificate is not lost within that window.\footnote{Appendix~\ref{app:benchmark-rdfe}
gives a complementary phase-based reference scheme that avoids this
continuation-specific stability condition.}
Localized regularity assumptions of this type are common in structured
evolving bandit models \citep{heidari2016tight,metelli2022stochastic,seznec2019rotting}.
\end{remark}

\begin{theorem}[Regret bound for TACC]
\label{thm:analysis-main-regret}
Suppose Assumption~\ref{ass:plugin} and Assumption~\ref{ass:mono-cert} hold,
and suppose the continuation cap satisfies
$S_0\lambda^{(L)}\le \lambda^{(H)}$. Then, with probability at least
$1-\delta$, Algorithm~\ref{alg:adaptive-mf-mab} satisfies
\begin{align}
R(\Lambda)
\le\;&
\sum_{k\in\mathcal{A}_\gamma}
\Delta_k
\left[
\lambda^{(L)}(N_\gamma+1)+\lambda^{(H)}
\right]
+
\sum_{k\in\mathcal{C}_\gamma^{\mathrm{d}}}
\Delta_k
\left[
\lambda^{(L)}(N_\gamma+S_0+1)+\lambda^{(H)}
\right]
\nonumber\\
&+
\sum_{k\in\mathcal{H}_\gamma}
\Delta_k
\left[
\lambda^{(L)}(N_\gamma+S_0+1)
+
\lambda^{(H)}
\left(
1+\left\lceil \frac{4\ell_\Lambda}{\Delta_k^2}\right\rceil
\right)
\right].
\label{eq:analysis-main-regret}
\end{align}
\end{theorem}

\paragraph{Proof sketch.}

On the event $\mathcal{E}_\Lambda$, the confidence bounds from both fidelities are uniformly valid for $\mu_k^{(H)}$. Then any selected suboptimal arm must remain plausible under the smaller upper confidence bound, which gives low- and high-fidelity selection certificates. The regret bound follows by counting pulls for each arm class. Arms in $\mathcal{A}_\gamma$ are certified before the nominal threshold, arms in $\mathcal{C}_\gamma^{\mathrm d}$ are certified through bounded low-fidelity continuation, and arms in $\mathcal{H}_\gamma$ require standard high-fidelity confirmation. Multiplying the resulting counts by $\lambda^{(m)}\Delta_k$ and summing over arms gives \eqref{eq:analysis-main-regret}.

Theorem~\ref{thm:analysis-main-regret} gives the finite-budget regret decomposition. The following Corollary~\ref{cor:analysis-main-regret-order} records its leading-order form by substituting $N_\gamma=\lceil \ell_\Lambda/\gamma^2\rceil$, using $\ell_\Lambda=O(\log T_\Lambda)$, and absorbing one-pull and bounded-continuation terms. The derivation is deferred to Appendix~\ref{app:proof-regret-Corollary}.
\begin{corollary}[Order-wise regret bound]
\label{cor:analysis-main-regret-order}
Under the conditions of Theorem~\ref{thm:analysis-main-regret}, for fixed $K$, $\rho$, and $\delta$, with probability at least $1-\delta$,
\begin{equation}
R(\Lambda)
=
O\!\left(
\lambda^{(L)}
\sum_{k\neq k^\star}
\frac{\Delta_k}{\gamma^2}\log T_\Lambda
+
\lambda^{(H)}
\sum_{k\in \mathcal{H}_\gamma}
\frac{\log T_\Lambda}{\Delta_k}
\right).
\label{eq:analysis-main-regret-order}
\end{equation}
\end{corollary}

The low-fidelity term in \eqref{eq:analysis-main-regret-order} is paid by all suboptimal arms. The logarithmic high-fidelity term is paid only by arms in $\mathcal{H}_\gamma$. Thus, the bound separates intermediate arms certified through low-fidelity continuation from those that still require high-fidelity confirmation.

\begin{remark}[Comparison with existing multi-fidelity rules]
\label{rem:analysis-mf-comparison}
An MF-UCB style static threshold rule \citep{kandasamy2016multi} escalates once $G_\Lambda(N_{k,t}^{(L)},\delta)<\gamma$. Thus, arms in $\mathcal{C}_\gamma$ require high-fidelity confirmation under the static rule, whereas TACC replaces this logarithmic confirmation with a bounded low-fidelity continuation cost for the detected subclass $\mathcal{C}_\gamma^{\mathrm d}$. The remaining arms are included in $\mathcal{H}_\gamma=\mathcal{B}_\gamma\cup\mathcal{C}_\gamma^{\mathrm u}$.

The comparison with \citet{wang2023revisited} is only at the level of logarithmic gap dependence. Their regret is measured in the highest-fidelity reward scale, whereas ours is cost weighted: a query of arm $k$ at fidelity $m$ incurs $\lambda^{(m)}\Delta_k$. Under the unit-cost normalization $\lambda^{(L)}=\lambda^{(H)}=1$, $T_\Lambda$ is order-equivalent to $\Lambda$, and Corollary~\ref{cor:analysis-main-regret-order} becomes
\[
R(\Lambda)
=
O\!\left(
\sum_{k\neq k^\star}
\frac{\Delta_k}{\gamma^2}\log \Lambda
+
\sum_{k\in \mathcal{H}_\gamma}
\frac{\log \Lambda}{\Delta_k}
\right).
\]
Thus, the high-fidelity confirmation term retains the usual logarithmic gap-dependent form, but is paid only by arms in $\mathcal{H}_\gamma$. This is a structural comparison of regret bounds, not a dominance statement under identical regret definitions. The formal static-threshold comparison is given in Proposition~\ref{prop:analysis-regret-improvement}.
\end{remark}
\section{Experiments}
\label{sec:exp}

We evaluate TACC in two settings. Synthetic experiments instantiate the model directly: low-fidelity feedback is cheap, initially mismatched with the high-fidelity target, and improves with repeated low-fidelity queries. The controlled LLM-as-a-judge experiment tests the same fidelity-allocation mechanism in a policy-evaluation pipeline, where each arm is a prompt-policy, low fidelity is weak-judge feedback, and high fidelity is verifier correctness. In both settings, we ask whether bounded post-threshold continuation can reduce cost-weighted regret before the learner escalates to high fidelity. The code is available at
\url{https://anonymous.4open.science/r/tacc-anonymous-release-CF04/}.

\subsection{Synthetic Bandit Experiments}
\label{subsec:syn}


\paragraph{Synthetic model and setting.}

We consider two settings. \textbf{Bandit Set A} contains $200$ arms with high-fidelity means sampled from $\mathcal{U}(0.1,0.9)$. The low-fidelity trajectory is generated with an arm-wise bias sign $s_k\in\{-1,1\}$ and an induced selected-average mismatch bound satisfying $B_k(n)=O(n^{-r})$. We set the initial mismatch scale to $\zeta_A=0.2$, costs to $\lambda^{(L)}=1$ and $\lambda^{(H)}=10$, and algorithm parameters to $\gamma=0.063$ and $S_0=10$. \textbf{Bandit Set B} contains $500$ arms with high-fidelity means sampled from $\mathcal{N}(0,1)$. It uses a larger mismatch scale $\zeta_B=1$ and a higher high-fidelity cost $\lambda^{(H)}=50$, with $\lambda^{(L)}=1$, $\gamma=0.141$, and $S_0=50$. This setting tests a harder regime with more arms, larger mismatch, and greater penalty for premature high-fidelity escalation. In both settings, we set the continuation parameter to $\eta=10^{-4}$. The main experiments use $r=0.5$; Appendix~\ref{app:synthetic-sensitivity} reports sensitivity to $r$.

\paragraph{Baselines and results.}
We compare TACC with \textsc{UCB}, which uses only high-fidelity feedback \citep{auer2002finite}; \textsc{MF-UCB}, a static multi-fidelity switching rule \citep{kandasamy2016multi}; and an LUCB-style static two-fidelity rule \citep{wang2023revisited}. We report cost-weighted cumulative pseudo-regret averaged over $10$ independent trials; Appendix~\ref{app:synthetic-sensitivity} gives the decay-rate sensitivity.

Figure~\ref{fig:simulation_plots} shows that TACC has lower average regret in both synthetic settings. The advantage is larger in Bandit Set B, where high-fidelity pulls are more costly. This is the regime where continuation has the clearest effect: avoiding even a modest number of unnecessary high-fidelity confirmations visibly reduces cost-weighted regret.

\begin{figure}[t]
    \centering
    \begin{subfigure}{0.49\columnwidth}
        \centering
        \includegraphics[width=\linewidth]{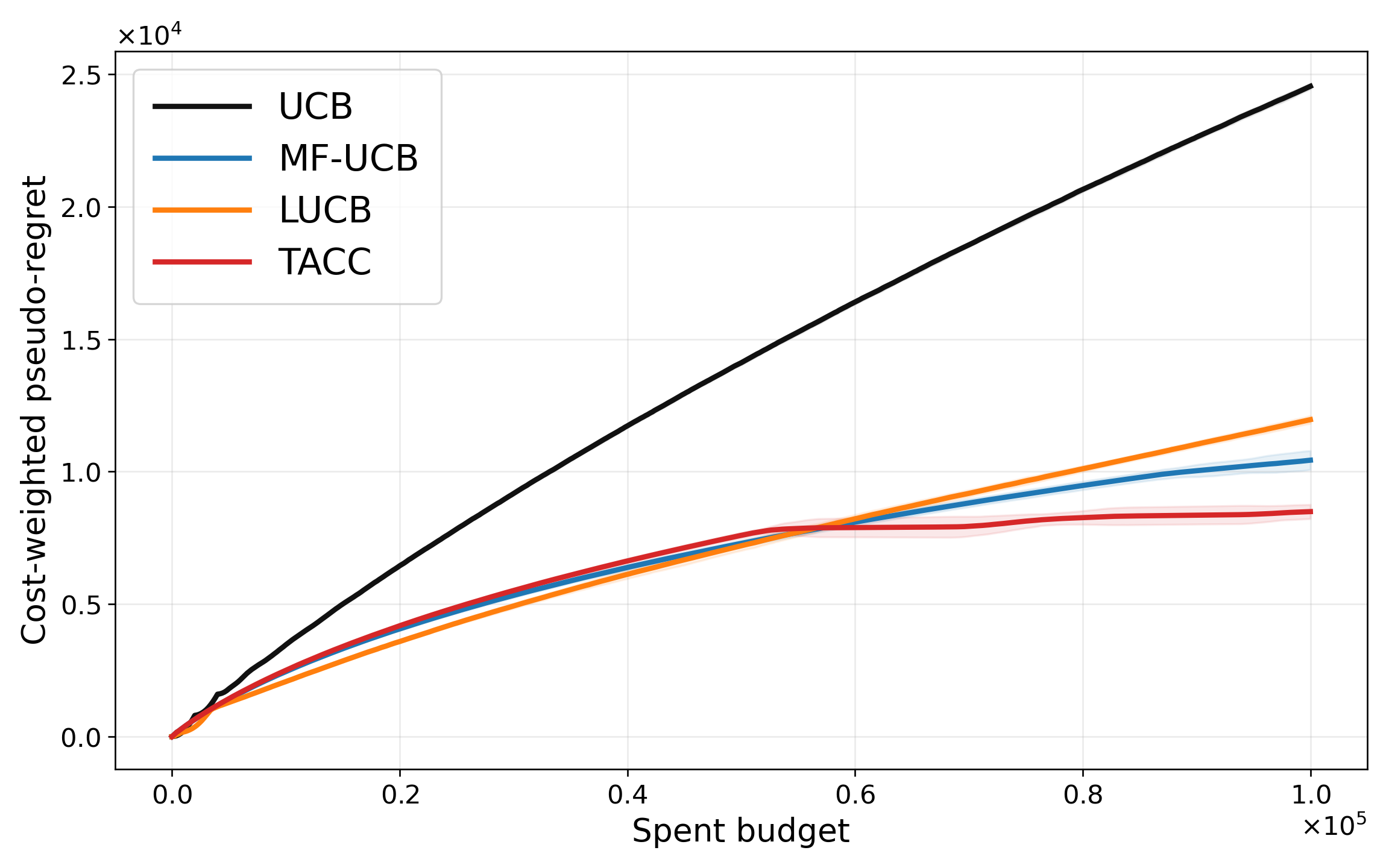}
        \caption{Bandit Set A, costs $(1,10)$}
        \label{fig:A_r05}
    \end{subfigure}
    \hfill
    \begin{subfigure}{0.49\columnwidth}
        \centering
        \includegraphics[width=\linewidth]{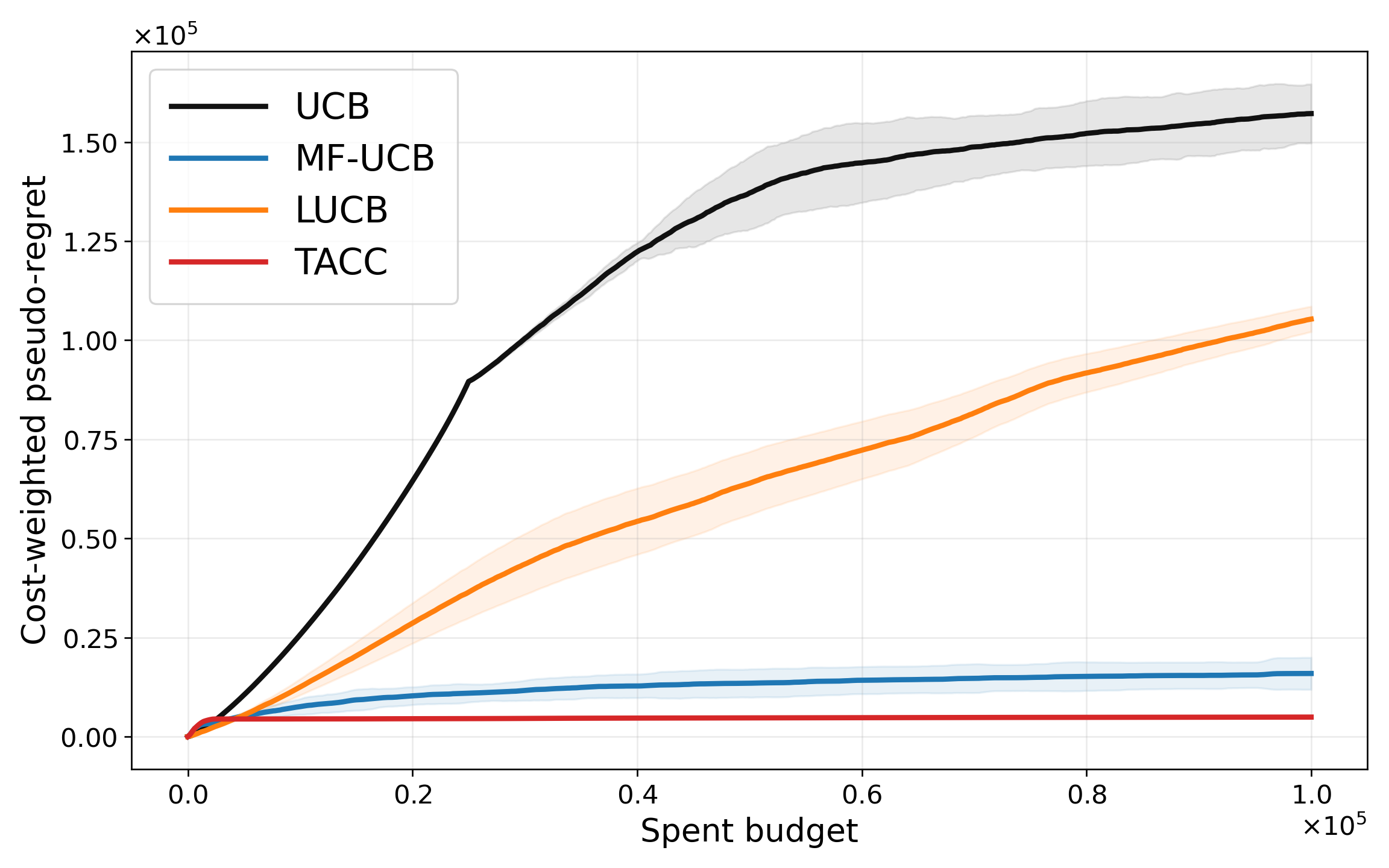}
        \caption{Bandit Set B, costs $(1,50)$}
        \label{fig:B_r05}
    \end{subfigure}
    \caption{Synthetic simulation results with $r=0.5$. TACC uses bounded low-fidelity continuation to reduce unnecessary high-fidelity queries.}
    \label{fig:simulation_plots}
\end{figure}

\subsection{LLM-as-a-Judge Experiment}
\label{subsec:llm}

We next test the same fidelity-allocation mechanism in a controlled LLM-as-a-judge evaluation task. Each arm is a prompt-policy for Natural Language Inference (NLI). A high-fidelity query checks whether the policy prediction matches the reference NLI label, while a low-fidelity query uses feedback from a weaker judge for the same policy. Candidate answers were generated using \texttt{gemini-3-flash-preview}. Low-fidelity judge signals were produced with \texttt{gemini-3.1-flash-lite-preview}; when an auxiliary stronger judge signal was used for diagnostics, we used \texttt{gemini-3-flash-preview}. The high-fidelity reward remains verifier correctness against
the reference NLI label.

\paragraph{Task and arms.}
Each item consists of a premise, a hypothesis, and a reference label in $\{\textsc{entailment},\textsc{neutral},\textsc{contradiction}\}$. A prompt-policy maps the premise--hypothesis pair to a predicted NLI label, and regret is measured against the best policy under verifier correctness. We use a fixed four-arm benchmark so that differences in regret mainly reflect fidelity allocation rather than large-arm exploration.

To instantiate the known mismatch bound setting in Assumption~\ref{ass:plugin}, we use a controlled weak-judge proxy. For arm $k$, after its $n$-th low-fidelity query,
\begin{equation}
    \mu_k^{(L)}(n)
    =
    \mu_k^{(H)}
    +
    b_k
    +
    a_k(n+n_0)^{-r}.
    \label{eq:llm-residual-bias-model}
\end{equation}
Here $n=N_{k,t}^{(L)}$ is the arm-specific low-fidelity count. The decaying term models the part of the weak-judge mismatch that improves with calibration, while $b_k$ allows a residual policy-level gap from the verifier. All mismatch correction-based methods are given the same valid bound induced by this proxy, but none observes the latent trajectory $\mu_k^{(L)}(n)$. Thus \eqref{eq:llm-residual-bias-model} is a stress test of the known selected-average mismatch-bound theory, not a claim that real LLM judges follow this exact law. The residual term is motivated by evidence that LLM judges can retain systematic mismatch even after improvement \citep{zheng2023judging,huang2025empiricaljudge}. Appendix~\ref{app:llm-pure-decay} removes $b_k$, and Appendix~\ref{app:llm-empirical-gap} measures the low--high gap from trained weak-judge checkpoints.

\paragraph{Baselines and protocol.}

For the LLM-as-a-judge experiments, we use two baselines (UCB and MF-UCB) introduced in the synthetic study (Section \ref{subsec:syn}) and add \textsc{DNC}, short for dynamic no-continuation, as the continuation ablation. Because this experiment is evaluated by cumulative regret rather than fixed-confidence best-arm identification, we do not use the LUCB-style elimination baseline as a primary LLM comparison. \textsc{DNC} uses the same aggregate optimistic score, count-dependent mismatch bound $B_k(n)$, and threshold $\gamma$ as TACC, but switches directly to high fidelity once the low-fidelity statistical radius falls below $\gamma$, without applying the mismatch-decrease continuation test. The comparison between TACC and \textsc{DNC} therefore isolates the effect of bounded continuation. The main figure reports the four primary methods, \textsc{TACC}, \textsc{DNC}, \textsc{MF-UCB}, and \textsc{UCB}; the diagnostic baselines \textsc{Weak-Fixed} and \textsc{Low-Improving} are reported in Appendix~\ref{app:llm-additional}. All methods use the same budget sequence and cost-weighted pseudo-regret metric, and paired comparisons use common random seeds.

\begin{wrapfigure}{r}{0.50\linewidth}
    \vspace{-1.0em}
    \centering
    \includegraphics[width=\linewidth]{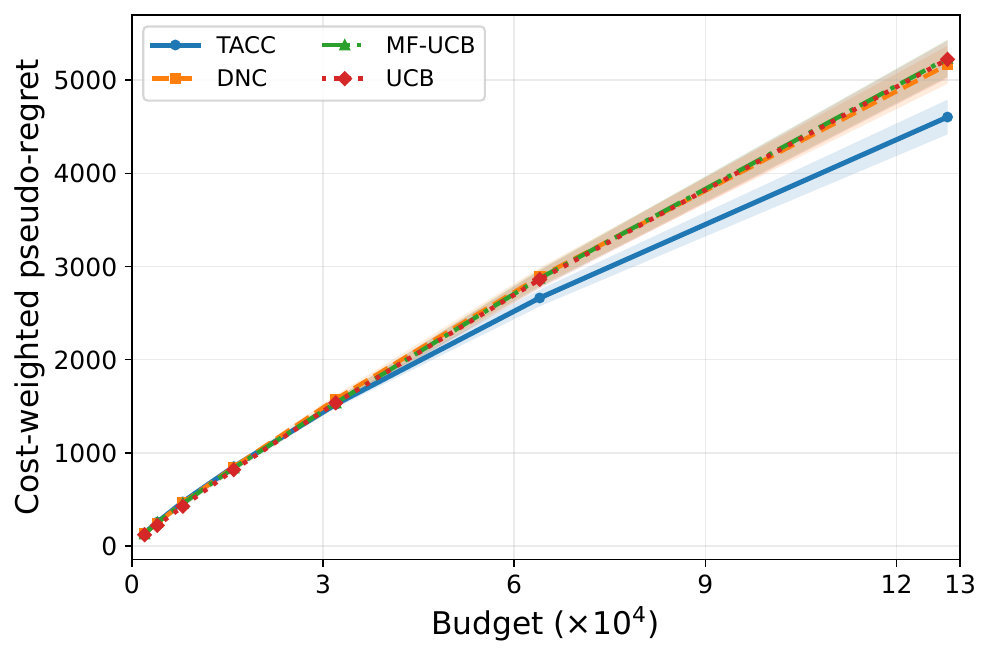}
    \vspace{-1.2em}
    \caption{LLM-as-a-judge result under residual low--high mismatch with $\lambda^{(H)}=500$. Curves show mean cost-weighted pseudo-regret over $200$ seeds.}
    \label{fig:llm_main}
    \vspace{-1.0em}
\end{wrapfigure}

\paragraph{Main result.}
The main residual-mismatch LLM-as-a-judge experiment uses $\lambda^{(L)}=1$, $\lambda^{(H)}=500$, $r=0.75$, $\zeta=0.4$, $b_k \equiv 0.05$, $\gamma=0.025$, $\eta=10^{-4}$, and $S_0=128$. We evaluate budgets up to $\Lambda=128000$ over $200$ independent seeds. Figure~\ref{fig:llm_main} shows the mean regret curves, and Table~\ref{tab:llm-main-summary} reports final-budget regret. TACC has the lowest final mean regret among the reported methods. The comparison with \textsc{DNC} is the direct ablation: both methods use the same improving-mismatch bound information, but only TACC spends a bounded number of post-threshold low-fidelity queries before escalating. The comparisons with \textsc{MF-UCB} and \textsc{UCB} show the benefit relative to static two-fidelity switching and high-fidelity-only evaluation.

\begin{table}[h]
\centering
\caption{
Final-budget regret in the residual-mismatch LLM-as-a-judge experiment with $\lambda^{(H)}=500$ and $\Lambda=128000$.
Values are mean cost-weighted pseudo-regret $\pm$ standard error over $200$ seeds.
Lower is better.
}
\label{tab:llm-main-summary}
\small
\setlength{\tabcolsep}{5pt}
\begin{tabular}{c c c c c}
\toprule
$\lambda^{(H)}$ & TACC & DNC & MF-UCB & UCB \\
\midrule
$500$
& $\mathbf{4023.4 \pm 247.3}$
& $5201.0 \pm 289.7$
& $5359.1 \pm 281.8$
& $5083.2 \pm 286.8$ \\
\bottomrule
\end{tabular}
\end{table}

\section{Conclusion}
This work investigates multi-fidelity bandits with improving proxy feedback through a canonical two-fidelity model. We use a selected-average mismatch bound to construct confidence intervals for the high-fidelity target, enabling TACC to replace high-fidelity queries with bounded low-fidelity continuation. The analysis characterizes when this substitution can reduce regret, and the experiments support this mechanism in synthetic simulations and finite-budget LLM-as-a-judge diagnostics.

\paragraph{Limitations.}
This work focuses on the canonical two-fidelity setting; extending TACC to multiple evolving proxies, under hierarchical or non-hierarchical fidelity structures, remains future work. In addition, it assumes a known selected-average mismatch bound, leaving data-driven estimation of this bound as an important direction. Finally, our experiments consider arm-local low-fidelity evolution, where the low-fidelity response improves only when the corresponding arm is queried. This does not capture global fine-tuning settings, where updating the low-fidelity model may affect all arms simultaneously. Extending TACC to such shared-evolution dynamics remains future work.


\newpage

\bibliographystyle{plainnat}
\bibliography{refs}

\clearpage

\appendix

\section{Notations}
\label{sec:appendix_notations}

\medskip
\begingroup
\setlength{\tabcolsep}{0pt}
\renewcommand{\arraystretch}{1.10}
\small
\begin{center}
\begin{tabularx}{\linewidth}{@{}l@{\hspace{0.9em}\vrule\hspace{0.9em}}X@{}}
\toprule
\(K,\,[K]\) & Number of arms; arm index set. \\
\(L,\,H\) & Low fidelity and high fidelity. \\
\(k^\star,\,\mu^\star\) & Optimal arm and its high-fidelity mean. \\
\(\mu_k^{(H)}\) & High-fidelity mean of arm \(k\). \\
\(\Delta_k\) & High-fidelity suboptimality gap of arm \(k\). \\
\(\lambda^{(L)},\,\lambda^{(H)}\) & Query costs of low and high fidelity. \\
\(\Lambda,\,T_\Lambda\) & Total budget; query upper bound induced by the budget. \\
\(C_t\) & Cumulative cost up to round \(t\). \\
\(R(\Lambda)\) & Cost-weighted pseudo-regret under budget \(\Lambda\). \\
\((I_t,m_t)\) & Arm and fidelity selected at round \(t\). \\
\(Y_{k,u}^{(H)},\,Y_{k,\tau}^{(L)}\) & High- and low-fidelity observations. \\
\(\mu_k^{(L)}(\tau)\) & Low-fidelity mean at the \(\tau\)-th low-fidelity query of arm \(k\). \\
\(\bar{\mu}_k^{(L)}(n)\) & Selected-average low-fidelity mean after \(n\) low-fidelity queries. \\
$U_k(n)$ & Cumulative mismatch certificate.\\
\(B_k(n)\) & Selected-average mismatch bound. \\
\(\Delta_k^{(L)}(n)\) & Effective low-fidelity gap. \\
\(\ell_\Lambda,\,G_\Lambda(n,\delta)\) & Budget-uniform log factor and confidence radius. \\
\(\operatorname{rad}_k^{(L)}(n,\delta),\,\operatorname{rad}^{(H)}(n,\delta)\) & Low- and high-fidelity confidence radii. \\
\(N_{k,t}^{(L)},\,N_{k,t}^{(H)}\) & Low- and high-fidelity query counts before round \(t\). \\
\(\hat{\mu}_{k,t}^{(L)},\,\hat{\mu}_{k,t}^{(H)}\) & Empirical means before round \(t\). \\
\(\mathrm{UCB}_{k,t}^{(m)},\,\mathrm{LCB}_{k,t}^{(m)}\) & Fidelity-\(m\) confidence bounds for arm \(k\). \\
\(\mathrm{UCB}_{k,t}\) & Aggregated optimistic score used for arm selection. \\
\(\gamma\) & Low-fidelity uncertainty threshold. \\
\(S_0\) & Continuation cap after crossing the threshold. \\
$A_k$ & Continuation pulls already spent on arm $k$.\\
$V_{k,t}(s)$ & Mismatch-bound decrease from $s$ additional low-fidelity queries.\\
\(N_\gamma\) & Nominal low-fidelity sampling threshold. \\
\(\tau_k(\gamma)\) & Low-fidelity certification time at resolution \(\gamma\). \\
\(A_\gamma,\,B_\gamma,\,C_\gamma\) & Easy, hard, and intermediate arm classes. \\
\(C_\gamma^{\mathrm{d}},\,C_\gamma^{\mathrm{u}}\) & Detected and undetected intermediate subclasses. \\
\(\mathcal{E}_\Lambda\) & Budget-uniform confidence event. \\
\bottomrule
\end{tabularx}
\end{center}
\renewcommand{\arraystretch}{1}
\endgroup
\section{Additional Proof Details and Algorithmic Variants}
\label{app:proofs}

This appendix contains the proofs of Lemma~\ref{lem:ci},
Propositions~\ref{prop:analysis-low-selection},
\ref{prop:analysis-high-selection}, and
\ref{prop:analysis-low-query-bound}, Theorem~\ref{thm:analysis-main-regret},
Corollary~\ref{cor:analysis-main-regret-order}, and
Proposition~\ref{prop:analysis-regret-improvement}.

\begin{lemma}[Valid confidence intervals]
\label{lem:ci}
Suppose Assumption~\ref{ass:plugin} holds. Then there exists an event
$\mathcal{E}_\Lambda$ with probability at least $1-\delta$ such that, on
$\mathcal{E}_\Lambda$, for all arms $k\in[K]$ and all rounds
$t\le T_\Lambda$,
\begin{equation}
\mu_k^{(H)}
\in
\bigl[
\mathrm{LCB}_{k,t}^{(L)},
\mathrm{UCB}_{k,t}^{(L)}
\bigr]
\quad\text{and}\quad
\mu_k^{(H)}
\in
\bigl[
\mathrm{LCB}_{k,t}^{(H)},
\mathrm{UCB}_{k,t}^{(H)}
\bigr].
\end{equation}
\end{lemma}

Throughout the appendix, we work on the event $\mathcal{E}_\Lambda$ from
Lemma~\ref{lem:ci}. Recall that
\[
\ell_\Lambda := \rho \log\left(\frac{2KT_\Lambda}{\delta}\right),
\qquad
G_\Lambda(n,\delta) := \sqrt{\frac{\ell_\Lambda}{n}},
\qquad
N_\gamma := \min\{n\ge 1:G_\Lambda(n,\delta)<\gamma\}.
\]
The constants inside the logarithm are not essential; they are chosen only to ensure a union bound over arms, fidelities, and all sample counts up to the budget horizon.

\paragraph{Proof of Lemma~\ref{lem:ci}.}
\label{app:proof-lem-ci}
For high fidelity, the claim is standard, for each fixed arm $k$ and sample size $n$, the empirical mean of the first $n$ high-fidelity samples is $1$-sub-Gaussian around $\mu_k^{(H)}$. By a uniform concentration bound over all $k\in[K]$ and all $1\le n\le T_\Lambda$, with probability at least $1-\delta/2$,
\begin{equation}
\left|
\hat{\mu}_{k,t}^{(H)}-\mu_k^{(H)}
\right|
\le
G_\Lambda\bigl(N_{k,t}^{(H)},\delta\bigr)
\end{equation}
for all arms $k$ and all rounds $t\le T_\Lambda$.

For low fidelity, recall the selected-average low-fidelity mean
\[
\bar{\mu}_k^{(L)}(n)
:=
\frac{1}{n}\sum_{\tau=1}^{n}\mu_k^{(L)}(\tau).
\]
The low-fidelity empirical mean concentrates around $\bar{\mu}_k^{(L)}(n)$, not around the
instantaneous mean $\mu_k^{(L)}(n)$. Using
\begin{equation}
\hat{\mu}_{k,t}^{(L)}-\mu_k^{(H)}
=
\bigl(
\hat{\mu}_{k,t}^{(L)}-\bar{\mu}_k^{(L)}(N_{k,t}^{(L)})
\bigr)
+
\bigl(
\bar{\mu}_k^{(L)}(N_{k,t}^{(L)})-\mu_k^{(H)}
\bigr),
\end{equation}
the first term is controlled by the same sub-Gaussian concentration argument, which yields,
with probability at least $1-\delta/2$,
\begin{equation}
\left|
\hat{\mu}_{k,t}^{(L)}-\bar{\mu}_k^{(L)}(N_{k,t}^{(L)})
\right|
\le
G_\Lambda\bigl(N_{k,t}^{(L)},\delta\bigr)
\end{equation}
for all arms $k$ and all rounds $t\le T_\Lambda$.

For the second term, Assumption~\ref{ass:plugin} gives
\begin{equation}
\left|
\bar{\mu}_k^{(L)}(n)-\mu_k^{(H)}
\right|
\le
\frac{1}{n}
\sum_{\tau=1}^{n}
\left|
\mu_k^{(L)}(\tau)-\mu_k^{(H)}
\right|
\le
\frac{U_k(n)}{n}
\le
B_k(n).
\end{equation}
Combining the stochastic and deterministic bounds gives
\begin{equation}
\left|
\hat{\mu}_{k,t}^{(L)}-\mu_k^{(H)}
\right|
\le
G_\Lambda\bigl(N_{k,t}^{(L)},\delta\bigr)
+
B_k\bigl(N_{k,t}^{(L)}\bigr).
\end{equation}
This is exactly the low-fidelity confidence interval. Taking a union bound over the high-fidelity and low-fidelity events proves that, with probability at least $1-\delta$, both intervals cover $\mu_k^{(H)}$ uniformly over all arms and all rounds within the budget horizon.

We next record the two selection conditions used in the regret proof.
Algorithm~\ref{alg:adaptive-mf-mab} selects arms by the aggregate score
$\mathrm{UCB}_{k,t}
=
\min\{\mathrm{UCB}_{k,t}^{(L)},\mathrm{UCB}_{k,t}^{(H)}\}$.
Thus a suboptimal arm can be selected only if neither fidelity-specific upper
bound has already certified it as suboptimal.

\begin{proposition}[Low-fidelity selection condition]
\label{prop:analysis-low-selection}
Suppose $\mathcal{E}_\Lambda$ holds. 
If a suboptimal arm $k\neq k^\star$ is selected at round $t$, then $2G_\Lambda\!\left(N_{k,t}^{(L)},\delta\right)
\ge
\Delta_k^{(L)}\!\left(N_{k,t}^{(L)}\right)$. Consequently, if $G_\Lambda\!\left(N_{k,t}^{(L)},\delta\right)<\gamma$, and $\Delta_k^{(L)}\!\left(N_{k,t}^{(L)}\right)\ge 2\gamma$, then arm $k$ is certified as suboptimal by low fidelity and will not be selected again.
\end{proposition}

\begin{proposition}[High-fidelity selection condition]
\label{prop:analysis-high-selection}
Suppose $\mathcal{E}_\Lambda$ holds. 
If a suboptimal arm $k\neq k^\star$ is selected and queried at high fidelity at round $t$, then $2G_\Lambda\!\left(N_{k,t}^{(H)},\delta\right)
\ge
\Delta_k$. Consequently,
\begin{equation}
N_{k,T_\Lambda}^{(H)}
\le
1+
\left\lceil
\frac{4\ell_\Lambda}{\Delta_k^2}
\right\rceil .
\label{eq:analysis-high-query-bound}
\end{equation}
\end{proposition}

\paragraph{Proof of Proposition~\ref{prop:analysis-low-selection}.}
\label{app:proof-prop-low-selection}
Fix a suboptimal arm $k\neq k^\star$ and suppose it is selected at round $t$. Since the
algorithm selects
\[
I_t \in \arg\max_{j\in[K]} UCB_{j,t},
\qquad
UCB_{j,t}:=\min_{m\in\{L,H\}}UCB_{j,t}^{(m)},
\]
we have
\begin{equation}
UCB_{k,t}\ge UCB_{k^\star,t}.
\end{equation}
On $\mathcal{E}_\Lambda$, both fidelity-specific confidence intervals for the optimal arm
cover $\mu_{k^\star}^{(H)}=\mu^\star$. Hence
\[
UCB_{k^\star,t}^{(L)}\ge \mu^\star,
\qquad
UCB_{k^\star,t}^{(H)}\ge \mu^\star,
\]
and therefore
\begin{equation}
UCB_{k^\star,t}
=
\min\{UCB_{k^\star,t}^{(L)},UCB_{k^\star,t}^{(H)}\}
\ge
\mu^\star.
\end{equation}
Thus
\begin{equation}
UCB_{k,t}\ge \mu^\star.
\end{equation}
Since $UCB_{k,t}$ is the minimum of the two fidelity-specific upper bounds, this implies in
particular that
\begin{equation}
UCB_{k,t}^{(L)}\ge \mu^\star.
\end{equation}
Let $n=N_{k,t}^{(L)}$. By the definition of the low-fidelity upper bound,
\begin{equation}
UCB_{k,t}^{(L)}
=
\hat{\mu}_{k,t}^{(L)}
+
G_\Lambda(n,\delta)
+
B_k(n).
\end{equation}
On $\mathcal{E}_\Lambda$, the empirical low-fidelity mean satisfies
\[
\hat{\mu}_{k,t}^{(L)}
\le
\bar{\mu}_k^{(L)}(n)
+
G_\Lambda(n,\delta).
\]
Therefore,
\begin{equation}
\mu^\star
\le
UCB_{k,t}^{(L)}
\le
\bar{\mu}_k^{(L)}(n)
+
2G_\Lambda(n,\delta)
+
B_k(n).
\end{equation}
Rearranging gives
\begin{equation}
2G_\Lambda(n,\delta)
\ge
\mu^\star-\bar{\mu}_k^{(L)}(n)-B_k(n)
=
\Delta_k^{(L)}(n).
\end{equation}
For the consequence, suppose that
\[
G_\Lambda\bigl(N_{k,t}^{(L)},\delta\bigr)<\gamma
\qquad\text{and}\qquad
\Delta_k^{(L)}\bigl(N_{k,t}^{(L)}\bigr)\ge 2\gamma.
\]
Then
\[
2G_\Lambda\bigl(N_{k,t}^{(L)},\delta\bigr)<2\gamma
\le
\Delta_k^{(L)}\bigl(N_{k,t}^{(L)}\bigr),
\]
which contradicts the necessary selection condition just proved. Hence arm $k$ cannot be
selected under these two inequalities. Since the arm is not selected, its low-fidelity count
does not increase, and the same certificate continues to rule out future selections. Thus the
arm is certified as suboptimal by low fidelity.

\paragraph{Proof of Proposition~\ref{prop:analysis-high-selection}.}
\label{app:proof-prop-high-selection}
Fix a suboptimal arm $k\neq k^\star$. Suppose arm $k$ is selected and queried at high
fidelity at round $t$. Since $k$ is selected, the same argument used in the proof of
Proposition~\ref{prop:analysis-low-selection} gives
\begin{equation}
UCB_{k,t}\ge \mu^\star.
\end{equation}
Because
\[
UCB_{k,t}=\min\{UCB_{k,t}^{(L)},UCB_{k,t}^{(H)}\},
\]
we must have
\begin{equation}
UCB_{k,t}^{(H)}\ge \mu^\star.
\end{equation}
Let $n=N_{k,t}^{(H)}$. By definition,
\[
UCB_{k,t}^{(H)}
=
\hat{\mu}_{k,t}^{(H)}
+
G_\Lambda(n,\delta).
\]
On $\mathcal{E}_\Lambda$,
\[
\hat{\mu}_{k,t}^{(H)}
\le
\mu_k^{(H)}
+
G_\Lambda(n,\delta).
\]
Therefore,
\begin{equation}
\mu^\star
\le
UCB_{k,t}^{(H)}
\le
\mu_k^{(H)}
+
2G_\Lambda(n,\delta).
\end{equation}
Rearranging gives
\begin{equation}
2G_\Lambda(n,\delta)\ge \mu^\star-\mu_k^{(H)}=\Delta_k.
\end{equation}
Since
\[
G_\Lambda(n,\delta)=\sqrt{\frac{\ell_\Lambda}{n}},
\]
the condition $2G_\Lambda(n,\delta)\ge \Delta_k$ implies
\begin{equation}
n
\le
\frac{4\ell_\Lambda}{\Delta_k^2}.
\end{equation}
At each high-fidelity query, $N_{k,t}^{(H)}$ denotes the number of previous high-fidelity
queries before the current one. Hence the total number of high-fidelity queries of arm $k$
is bounded by
\begin{equation}
N_{k,T_\Lambda}^{(H)}
\le
1+\left\lceil
\frac{4\ell_\Lambda}{\Delta_k^2}
\right\rceil,
\end{equation}
where the additive constant accounts for initialization and the boundary query.

We next bound the number of low-fidelity queries. The bound follows from the
threshold rule before $N_\gamma$ and from the continuation cap after the
threshold crossing.
For an arm $k\in\mathcal{C}_\gamma$, define the low-fidelity hitting time
\begin{equation}
T_k^\tau
:=
\inf\left\{
t\ge 1:
N_{k,t}^{(L)}\ge \tau_k(\gamma)
\right\},
\label{eq:app-analysis-hit-time}
\end{equation}
with the convention that $T_k^\tau=\infty$ if the set is empty. Since
$N_{k,t}^{(L)}$ denotes the number of low-fidelity queries before round $t$, an
arm that reaches $\tau_k(\gamma)$ after an allowed query at round $T_\Lambda$
has hitting time $T_\Lambda+1$ under this convention. We define
\begin{equation}
\begin{aligned}
\mathcal{C}_\gamma^{\mathrm d}
:=
\Bigl\{
k\in\mathcal{C}_\gamma:\;&
T_k^\tau\le T_\Lambda+1,\\
&
\sum_{1\le t<T_k^\tau}
\mathbf{1}\!\left\{
I_t=k,\,
N_\gamma\le N_{k,t}^{(L)}<\tau_k(\gamma),\,
m_t=H
\right\}
=0
\Bigr\}.
\end{aligned}
\label{eq:app-analysis-Cdet}
\end{equation}
Finally, let
$\mathcal{C}_\gamma^{\mathrm u}
:=
\mathcal{C}_\gamma\setminus\mathcal{C}_\gamma^{\mathrm d}$.
Thus an arm in $\mathcal{C}_\gamma^{\mathrm d}$ is not merely potentially
certifiable within the continuation window: along the realized execution, its
low-fidelity count reaches $\tau_k(\gamma)$, and the arm is not queried at high
fidelity while its low-fidelity count lies between $N_\gamma$ and
$\tau_k(\gamma)$.

\begin{proposition}[Low-fidelity query bound]
\label{prop:analysis-low-query-bound}
Suppose $\mathcal{E}_\Lambda$ holds. Under Algorithm~\ref{alg:adaptive-mf-mab},
every suboptimal arm $k\neq k^\star$ satisfies
\begin{equation}
N_{k,T_\Lambda}^{(L)}
\le
N_\gamma+S_0+1,
\label{eq:analysis-low-query-bound}
\end{equation}
where the additive constant accounts for initialization and boundary effects.
Moreover:
\begin{enumerate}[label=(\roman*)]
    \item if $k\in\mathcal{A}_\gamma$, then $k$ is certified by low fidelity at
    the nominal switching threshold;
    \item if $k\in\mathcal{C}_\gamma^{\mathrm d}$, then the realized execution
    reaches the low-fidelity certification count $\tau_k(\gamma)$ within the
    continuation window, and the arm is certified by low fidelity.
\end{enumerate}
\end{proposition}

\begin{proof}
Fix a suboptimal arm $k\neq k^\star$. Low fidelity can be queried before the
threshold crossing only while
$G_\Lambda(N_{k,t}^{(L)},\delta)\ge \gamma$. By the definition of $N_\gamma$,
this implies $N_{k,t}^{(L)}<N_\gamma$. Hence the pre-threshold phase contributes
at most $N_\gamma+1$ low-fidelity pulls, including the boundary effect from
pre-round counts. After the threshold crossing, low fidelity can be queried only
through the continuation rule, whose counter is capped by $S_0$. Therefore
\[
N_{k,T_\Lambda}^{(L)}
\le
N_\gamma+S_0+1.
\]

If $k\in\mathcal{A}_\gamma$, then $\tau_k(\gamma)\le N_\gamma$. By
Assumption~\ref{ass:mono-cert}, the inequality
$\Delta_k^{(L)}(n)\ge 2\gamma$ remains valid through the relevant continuation
window once it first holds. In particular, when the low-fidelity radius first
falls below $\gamma$, we have both
$G_\Lambda(N_{k,t}^{(L)},\delta)<\gamma$ and
$\Delta_k^{(L)}(N_{k,t}^{(L)})\ge 2\gamma$. Proposition~\ref{prop:analysis-low-selection}
then rules out further selection of arm $k$.

If $k\in\mathcal{C}_\gamma^{\mathrm d}$, then
$T_k^\tau\le T_\Lambda+1$ by \eqref{eq:app-analysis-Cdet}, so the realized
execution reaches the count $\tau_k(\gamma)$ within the budget horizon, up to
the pre-round convention. The same definition also rules out high-fidelity
queries of arm $k$ while
$N_\gamma\le N_{k,t}^{(L)}<\tau_k(\gamma)$. Since
$\tau_k(\gamma)\le N_\gamma+S_0$ for $k\in\mathcal{C}_\gamma$, the certification
count is reached within the continuation window. At that count,
$\Delta_k^{(L)}(\tau_k(\gamma))\ge 2\gamma$; Assumption~\ref{ass:mono-cert}
keeps the certificate valid within the window, and
Proposition~\ref{prop:analysis-low-selection} rules out further selection.
\end{proof}

\paragraph{Proof of Theorem~\ref{thm:analysis-main-regret}.}
\label{app:proof-thm-tacc-regret}
On $\mathcal{E}_\Lambda$, Lemma~\ref{lem:ci} ensures that all confidence intervals are valid
uniformly over the budget horizon. The regret of querying a suboptimal arm $k$ at fidelity
$m\in\{L,H\}$ is
\(
\lambda^{(m)}\Delta_k.
\)
Therefore,
\begin{equation}
R(\Lambda)
=
\sum_{k\neq k^\star}
\Delta_k
\left[
\lambda^{(L)}N_{k,T_\Lambda}^{(L)}
+
\lambda^{(H)}N_{k,T_\Lambda}^{(H)}
\right],
\end{equation}
where the optimal arm contributes zero pseudo-regret.

For $k\in\mathcal{A}_\gamma$, low fidelity certifies suboptimality by the
nominal switching threshold because $\tau_k(\gamma)\le N_\gamma$. By
Proposition~\ref{prop:analysis-low-query-bound}, such an arm uses at most
$N_\gamma+1$ low-fidelity pulls, including the boundary effect, and does not
incur the logarithmic high-fidelity confirmation term. Keeping one
high-fidelity initialization or reference query gives
\begin{equation}
R_k(\Lambda)
\le
\Delta_k
\left[
\lambda^{(L)}(N_\gamma+1)+\lambda^{(H)}
\right].
\end{equation}

For $k\in\mathcal{C}_\gamma^{\mathrm d}$, the arm is not certified at the
static threshold, but the realized execution reaches $\tau_k(\gamma)$ through
low-fidelity continuation without a high-fidelity query in the intermediate
interval specified by \eqref{eq:app-analysis-Cdet}. Since
$\tau_k(\gamma)\le N_\gamma+S_0$, Proposition~\ref{prop:analysis-low-query-bound}
implies certification within at most $N_\gamma+S_0+1$ low-fidelity pulls. Hence
\begin{equation}
R_k(\Lambda)
\le
\Delta_k
\left[
\lambda^{(L)}(N_\gamma+S_0+1)+\lambda^{(H)}
\right].
\end{equation}

For $k\in\mathcal{H}_\gamma=\mathcal{B}_\gamma\cup\mathcal{C}_\gamma^{\mathrm u}$,
the analysis does not credit the arm with realized low-fidelity continuation
certification. The low-fidelity count is still bounded by
Proposition~\ref{prop:analysis-low-query-bound}, while the number of
high-fidelity queries is bounded by Proposition~\ref{prop:analysis-high-selection}:
\[
N_{k,T_\Lambda}^{(L)}
\le
N_\gamma+S_0+1,
\qquad
N_{k,T_\Lambda}^{(H)}
\le
1+\left\lceil\frac{4\ell_\Lambda}{\Delta_k^2}\right\rceil .
\]
Therefore
\begin{equation}
R_k(\Lambda)
\le
\Delta_k
\left[
\lambda^{(L)}(N_\gamma+S_0+1)
+
\lambda^{(H)}
\left(
1+\left\lceil \frac{4\ell_\Lambda}{\Delta_k^2}\right\rceil
\right)
\right].
\end{equation}

Summing these three bounds over the corresponding arm classes yields
\begin{align}
R(\Lambda)
\le\;&
\sum_{k\in\mathcal{A}_\gamma}
\Delta_k
\left[
\lambda^{(L)}(N_\gamma+1)
+
\lambda^{(H)}
\right]
\nonumber\\
&+
\sum_{k\in\mathcal{C}_\gamma^{\mathrm{d}}}
\Delta_k
\left[
\lambda^{(L)}(N_\gamma+S_0+1)
+
\lambda^{(H)}
\right]
\nonumber\\
&+
\sum_{k\in\mathcal{B}_\gamma\cup\mathcal{C}_\gamma^{\mathrm{u}}}
\Delta_k
\left[
\lambda^{(L)}(N_\gamma+S_0+1)
+
\lambda^{(H)}
\left(
1+\left\lceil
\frac{4\ell_\Lambda}{\Delta_k^2}
\right\rceil
\right)
\right].
\end{align}

This is the claimed regret bound.

\paragraph{Proof of Corollary~\ref{cor:analysis-main-regret-order}.}
\label{app:proof-regret-Corollary}
The result follows by simplifying the finite-sample bound in
Theorem~\ref{thm:analysis-main-regret}. Since
$N_\gamma\le 1+\ell_\Lambda/\gamma^2$, the low-fidelity contribution of each
suboptimal arm is
\[
O\!\left(
\lambda^{(L)}\Delta_k
\left(
\frac{\ell_\Lambda}{\gamma^2}+S_0+1
\right)
\right).
\]
For arms in $\mathcal{A}_\gamma\cup\mathcal{C}_\gamma^{\mathrm d}$, the
high-fidelity contribution is only the initialization or reference cost. The
logarithmic high-fidelity confirmation term appears only for arms in
$\mathcal{H}_\gamma$. For each such arm,
\[
\Delta_k\lambda^{(H)}
\left(
1+\left\lceil \frac{4\ell_\Lambda}{\Delta_k^2}\right\rceil
\right)
=
O\!\left(
\lambda^{(H)}\frac{\ell_\Lambda}{\Delta_k}
\right)
+
O(\lambda^{(H)}\Delta_k).
\]
Using $\ell_\Lambda=\rho\log(2KT_\Lambda/\delta)$ gives
\eqref{eq:analysis-main-regret-order}.

\begin{proposition}[Arm-level comparison with the static threshold rule]
\label{prop:analysis-regret-improvement}
For any arm $k\in\mathcal{C}_\gamma^{\mathrm d}$, the static and adaptive
arm-level upper bounds satisfy
\begin{equation}
R_{k,\mathrm{static}}(\Lambda)-R_{k,\mathrm{adapt}}(\Lambda)
\ge
\Delta_k
\left[
\lambda^{(H)}
\left\lceil
\frac{4\ell_\Lambda}{\Delta_k^2}
\right\rceil
-
\lambda^{(L)}S_0
\right].
\label{eq:regret_compare}
\end{equation}
In particular, if $S_0\lambda^{(L)}\le \lambda^{(H)}$ and
$\left\lceil 4\ell_\Lambda/\Delta_k^2\right\rceil\ge 2$, then the adaptive
rule yields a strictly smaller arm-level regret upper bound for arm $k$.
\end{proposition}



\begin{proof}
Fix $k\in\mathcal{C}_\gamma^{\mathrm d}$. Under the static threshold rule, low
fidelity is queried only until
$G_\Lambda(N_{k,t}^{(L)},\delta)<\gamma$. Since $k\in\mathcal{C}_\gamma$, we
have $\tau_k(\gamma)>N_\gamma$, so the low-fidelity certificate is not available
at the nominal switching threshold. The static rule therefore treats this arm as
requiring high-fidelity confirmation. By Proposition~\ref{prop:analysis-high-selection},
its arm-level static upper bound is
\begin{equation}
R_{k,\mathrm{static}}(\Lambda)
\le
\Delta_k
\left[
\lambda^{(L)}(N_\gamma+1)
+
\lambda^{(H)}
\left(
1+\left\lceil\frac{4\ell_\Lambda}{\Delta_k^2}\right\rceil
\right)
\right].
\end{equation}

For the adaptive rule, $k\in\mathcal{C}_\gamma^{\mathrm d}$ means that the
realized execution reaches $\tau_k(\gamma)$ through low-fidelity continuation
without high-fidelity interruption in the intermediate interval. By
Proposition~\ref{prop:analysis-low-query-bound}, the corresponding adaptive
upper bound is
\begin{equation}
R_{k,\mathrm{adapt}}(\Lambda)
\le
\Delta_k
\left[
\lambda^{(L)}(N_\gamma+S_0+1)+\lambda^{(H)}
\right].
\end{equation}
Subtracting the two upper bounds yields the claimed comparison:
\begin{equation}
R_{k,\mathrm{static}}(\Lambda)-R_{k,\mathrm{adapt}}(\Lambda)
\ge
\Delta_k
\left[
\lambda^{(H)}
\left\lceil\frac{4\ell_\Lambda}{\Delta_k^2}\right\rceil
-
\lambda^{(L)}S_0
\right].
\end{equation}
If $S_0\lambda^{(L)}\le\lambda^{(H)}$ and
$\left\lceil 4\ell_\Lambda/\Delta_k^2\right\rceil\ge 2$, the right-hand side is
strictly positive.
\end{proof}

\section{A complementary oracle-bound reference scheme without continuation assumptions}
\label{app:benchmark-rdfe}

The main-text analysis of TACC is continuation specific: it uses additional structural
conditions to certify when a short post-threshold low-fidelity continuation block is
provably aligned with the elimination of suboptimal arms. To separate this layer from the
underlying complexity of fidelity choice itself, we record here a complementary
oracle-bound reference scheme based on phase-wise resolution targeting.

This reference scheme is not meant to replace the main-text policy. Its role is to isolate
the resolution-dependent certification principle induced by the selected-average bound
model without invoking the continuation-specific assumptions used in the TACC analysis. In
this sense, it serves as a structural reference point for the self-improving two-fidelity
problem, while TACC remains the primary online policy.

For a target resolution $\varepsilon>0$, define the low-fidelity certification cost for arm
$k$ by
\begin{equation}
\mathsf{c}_{k}^{(L)}(\varepsilon;\Lambda)
:=
\lambda^{(L)}
\min\Bigl\{
n\ge 1:
G_\Lambda(n,\delta)+B_k(n)\le \varepsilon/8
\Bigr\},
\end{equation}
and the high-fidelity certification cost by
\begin{equation}
\mathsf{c}^{(H)}(\varepsilon;\Lambda)
:=
\lambda^{(H)}
\min\Bigl\{
n\ge 1:
G_\Lambda(n,\delta)\le \varepsilon/8
\Bigr\}.
\end{equation}
With the convention that the minimum of an empty set is $+\infty$, define the oracle cost
complexity
\begin{equation}
\mathsf{c}_k^\star(\varepsilon;\Lambda)
:=
\min\bigl\{
\mathsf{c}_{k}^{(L)}(\varepsilon;\Lambda),\,
\mathsf{c}^{(H)}(\varepsilon;\Lambda)
\bigr\}.
\end{equation}

The reference scheme uses a dyadic resolution schedule
\begin{equation}
\varepsilon_r := 2^{-r},
\qquad
r=1,2,\dots,
\end{equation}
and maintains an active set $\mathcal{S}_r$ of arms not yet certified as suboptimal.

\begin{algorithm}[t]
\caption{Resolution-Dependent Fidelity Elimination (RDFE)}
\label{alg:app-rdfe}
\begin{algorithmic}[1]
\Require number of arms $K$, costs $(\lambda^{(L)},\lambda^{(H)})$, confidence level $\delta$, budget $\Lambda$
\State query each arm once at low fidelity and once at high fidelity
\State initialize empirical means, counts, and $\Lambda_{\mathrm{used}}\gets K(\lambda^{(L)}+\lambda^{(H)})$
\State set $r\gets 1$, $\mathcal{S}_1\gets [K]$, and let
$\widehat{k}_{\mathrm{out}}$ be the warm-start LCB leader
\While{$|\mathcal{S}_r|>1$}
    \State $\varepsilon_r\gets 2^{-r}$
    \For{each arm $k\in\mathcal{S}_r$}
        \State choose
        \(m_{k,r}\in \operatorname*{arg\,min}
        \{\mathsf{c}_{k}^{(L)}(\varepsilon_r;\Lambda),
        \mathsf{c}^{(H)}(\varepsilon_r;\Lambda)\}\)
        \If{$m_{k,r}=L$}
            \State query arm $k$ at low fidelity until
            \(G_\Lambda(N_k^{(L)},\delta)+B_k(N_k^{(L)})\le \varepsilon_r/8\),
            or until the remaining budget cannot pay for the next low-fidelity query
            \If{the stopping condition is not reached}
                \State \Return $\widehat{k}_{\mathrm{out}}$
            \EndIf
            \State set
            \(\mathrm{UCB}_{k,r}\gets
            \hat{\mu}_k^{(L)}+G_\Lambda(N_k^{(L)},\delta)+B_k(N_k^{(L)})\)
            and
            \(\mathrm{LCB}_{k,r}\gets
            \hat{\mu}_k^{(L)}-G_\Lambda(N_k^{(L)},\delta)-B_k(N_k^{(L)})\)
        \Else
            \State query arm $k$ at high fidelity until
            \(G_\Lambda(N_k^{(H)},\delta)\le \varepsilon_r/8\),
            or until the remaining budget cannot pay for the next high-fidelity query
            \If{the stopping condition is not reached}
                \State \Return $\widehat{k}_{\mathrm{out}}$
            \EndIf
            \State set
            \(\mathrm{UCB}_{k,r}\gets
            \hat{\mu}_k^{(H)}+G_\Lambda(N_k^{(H)},\delta)\)
            and
            \(\mathrm{LCB}_{k,r}\gets
            \hat{\mu}_k^{(H)}-G_\Lambda(N_k^{(H)},\delta)\)
        \EndIf
    \EndFor
    \State choose
    \(\hat{k}_r\in\operatorname*{arg\,max}_{k\in\mathcal{S}_r}
    \mathrm{LCB}_{k,r}\)
    \State update
    \(\mathcal{S}_{r+1}\gets
    \{k\in\mathcal{S}_r:
    \mathrm{UCB}_{k,r}\ge
    \mathrm{LCB}_{\hat{k}_r,r}-\varepsilon_r/4\}\)
    \State set
    \(\widehat{k}_{\mathrm{out}}\in
    \operatorname*{arg\,max}_{k\in\mathcal{S}_{r+1}}\mathrm{LCB}_{k,r}\)
    \State $r\gets r+1$
\EndWhile
\State \Return the unique arm in $\mathcal{S}_r$ if $|\mathcal{S}_r|=1$; otherwise return $\widehat{k}_{\mathrm{out}}$
\end{algorithmic}
\end{algorithm}

The scheme is deliberately phase based. Its role is not to provide the most competitive
finite-budget online policy, but to isolate the resolution-dependent complexity object
induced by the shrinking selected-average mismatch bound.

The analysis reuses the confidence-validity event $\mathcal{E}_\Lambda$ from
Lemma~\ref{lem:ci}. Unlike the main-text TACC analysis, it does not invoke
continuation-specific assumptions. The proof has three steps: once a phase is completed,
every active arm has confidence width at most $\varepsilon_r/4$; any suboptimal arm with
gap at least $\varepsilon_r$ is eliminated at that phase; and the additional phase cost is
charged to the corresponding oracle certification cost
$\mathsf{c}_k^\star(\varepsilon_r;\Lambda)$.

For bookkeeping, call a phase $r$ \emph{completed} if the algorithm finishes processing
every arm in $\mathcal{S}_r$ and performs the elimination update that defines
$\mathcal{S}_{r+1}$. Let $R_\Lambda$ denote the number of completed phases under budget
$\Lambda$. For each arm $k$ and completed phase $r\le R_\Lambda$, let
\begin{equation}
\Delta \Lambda_{k,r}
\end{equation}
denote the additional cost spent on arm $k$ during phase $r$, with
$\Delta \Lambda_{k,r}:=0$ if $k\notin\mathcal{S}_r$.

\begin{lemma}[Phase-end confidence width for RDFE]
\label{lem:app-rdfe-phase-width}
Suppose phase $r$ is completed and the event $\mathcal{E}_\Lambda$ holds. Then for every arm $k\in\mathcal{S}_r$,
\begin{equation}
\mu_k^{(H)}\in[\mathrm{LCB}_{k,r},\mathrm{UCB}_{k,r}]
\qquad\text{and}\qquad
\mathrm{UCB}_{k,r}-\mathrm{LCB}_{k,r}\le \varepsilon_r/4.
\end{equation}
Consequently,
\begin{equation}
\mathrm{LCB}_{k,r}\ge \mu_k^{(H)}-\varepsilon_r/4,
\qquad
\mathrm{UCB}_{k,r}\le \mu_k^{(H)}+\varepsilon_r/4.
\end{equation}
\end{lemma}

\begin{proof}
The claim follows directly from the stopping rule of Algorithm~\ref{alg:app-rdfe} and
Lemma~\ref{lem:ci}. If low fidelity is chosen, then
\[
G_\Lambda(N_k^{(L)},\delta)+B_k(N_k^{(L)})\le \varepsilon_r/8,
\]
so the resulting interval width is at most $\varepsilon_r/4$. If high fidelity is chosen,
then
\[
G_\Lambda(N_k^{(H)},\delta)\le \varepsilon_r/8,
\]
and the same conclusion follows. Coverage is inherited from Lemma~\ref{lem:ci}.
\end{proof}

\begin{lemma}[Phase elimination for RDFE]
\label{lem:app-rdfe-elim}
Suppose phase $r$ is completed and the event $\mathcal{E}_\Lambda$ holds. If $\Delta_k\ge \varepsilon_r$, then arm $k$ is eliminated by the end of phase $r$.
\end{lemma}

\begin{proof}
By Lemma~\ref{lem:app-rdfe-phase-width},
\[
\mathrm{UCB}_{k,r}\le \mu_k^{(H)}+\varepsilon_r/4.
\]
For the optimal arm $k^\star$,
\[
\mathrm{LCB}_{k^\star,r}\ge \mu^\star-\varepsilon_r/4.
\]
Since $\hat{k}_r$ maximizes the lower confidence bound over $\mathcal{S}_r$,
\[
\mathrm{LCB}_{\hat{k}_r,r}\ge \mathrm{LCB}_{k^\star,r}\ge \mu^\star-\varepsilon_r/4,
\]
and hence
\[
\mathrm{LCB}_{\hat{k}_r,r}-\varepsilon_r/4 \ge \mu^\star-\varepsilon_r/2.
\]
If $\Delta_k=\mu^\star-\mu_k^{(H)}\ge \varepsilon_r$, then
\[
\mathrm{UCB}_{k,r}
\le
\mu_k^{(H)}+\varepsilon_r/4
=
\mu^\star-\Delta_k+\varepsilon_r/4
\le
\mu^\star-\frac{3\varepsilon_r}{4}
<
\mu^\star-\frac{\varepsilon_r}{2}
\le
\mathrm{LCB}_{\hat{k}_r,r}-\varepsilon_r/4.
\]
Thus arm $k$ fails the survival test and is eliminated.
\end{proof}

\begin{lemma}[Single-phase additional cost upper bound for RDFE]
\label{lem:app-rdfe-phase-cost}
Fix a completed phase $r\le R_\Lambda$ and an arm $k\in\mathcal{S}_r$. Then
\begin{equation}
\Delta\Lambda_{k,r}
\le
\mathsf{c}_k^\star(\varepsilon_r;\Lambda).
\end{equation}
\end{lemma}

\begin{proof}
If phase $r$ chooses low fidelity, the algorithm continues querying until
\[
G_\Lambda(N_k^{(L)},\delta)+B_k(N_k^{(L)})\le \varepsilon_r/8.
\]
Hence the additional low-fidelity cost incurred during phase $r$ is at most the
from-scratch low-fidelity certification cost
$\mathsf{c}_k^{(L)}(\varepsilon_r;\Lambda)$. If phase $r$ chooses high fidelity, the same
argument gives an upper bound by $\mathsf{c}^{(H)}(\varepsilon_r;\Lambda)$. Since the phase
chooses the cheaper of the two fidelities, the additional phase cost is at most
$\mathsf{c}_k^\star(\varepsilon_r;\Lambda)$.
\end{proof}

\begin{assumption}[Dyadic regularity of oracle costs for the reference scheme]
\label{ass:app-rdfe-dyadic}
For each suboptimal arm $k\neq k^\star$, there exists a constant $C_{\mathrm{dyad}}>0$ such that
\begin{equation}
\sum_{r\le r_k}\mathsf{c}_k^\star(\varepsilon_r;\Lambda)
\le
C_{\mathrm{dyad}}\,\mathsf{c}_k^\star(\varepsilon_{r_k};\Lambda),
\qquad
r_k:=\min\{r:\varepsilon_r\le \Delta_k/2\}.
\end{equation}
\end{assumption}

\begin{remark}
Assumption~\ref{ass:app-rdfe-dyadic} is a dyadic summability condition on the oracle
certification cost. It holds whenever $\mathsf{c}_k^\star(\varepsilon;\Lambda)$ grows
regularly as $\varepsilon$ decreases, for example under the polynomial-type oracle costs
induced by the high-fidelity term
$\mathsf{c}^{(H)}(\varepsilon;\Lambda)\asymp \lambda^{(H)}\log(T_\Lambda)/\varepsilon^2$
and by the common power-mismatch bound regime $B_k(n)\lesssim \Gamma_k n^{-\alpha}$. Its role is
to exclude highly oscillatory oracle-cost profiles across dyadic resolutions, so that the
cumulative phase cost is controlled by the terminal certification scale.
\end{remark}

\begin{theorem}[Regret guarantee for RDFE]
\label{thm:app-rdfe-regret}
Under Assumption~\ref{ass:plugin} and
Assumption~\ref{ass:app-rdfe-dyadic}, the reference scheme
Algorithm~\ref{alg:app-rdfe} satisfies, with probability at least $1-\delta$,
\begin{equation}
R_{\mathrm{RDFE}}(\Lambda)
\le
C_1\sum_{k\neq k^\star}\Delta_k\,\mathsf{c}_k^\star(\Delta_k/4;\Lambda)
+
C_2K(\lambda^{(L)}+\lambda^{(H)}),
\end{equation}
for some absolute constants $C_1,C_2>0$.
\end{theorem}

\begin{proof}
Fix a suboptimal arm $k\neq k^\star$, and let
\[
r_k:=\min\{r:\varepsilon_r\le \Delta_k/2\}.
\]
By Lemma~\ref{lem:app-rdfe-elim}, arm $k$ cannot survive beyond phase $r_k$. Therefore its
total additional phase-based exploration cost is bounded by
\[
\sum_{r\le r_k}\Delta\Lambda_{k,r}.
\]
Applying Lemma~\ref{lem:app-rdfe-phase-cost} phase by phase gives
\[
\sum_{r\le r_k}\Delta\Lambda_{k,r}
\le
\sum_{r\le r_k}\mathsf{c}_k^\star(\varepsilon_r;\Lambda).
\]
Assumption~\ref{ass:app-rdfe-dyadic} then implies
\[
\sum_{r\le r_k}\mathsf{c}_k^\star(\varepsilon_r;\Lambda)
\le
C\,\mathsf{c}_k^\star(\Delta_k/4;\Lambda).
\]
Since every unit of cost spent on arm $k$ contributes at most $\Delta_k$ to the
cost-weighted pseudo-regret,
\[
R_k(\Lambda)
\le
C\,\Delta_k\,\mathsf{c}_k^\star(\Delta_k/4;\Lambda).
\]
Summing over all suboptimal arms yields the main term. The warm start contributes at most
one low-fidelity and one high-fidelity pull per arm, which gives the additive lower-order
term.
\end{proof}

\section{Additional Experimental Details}
\label{app:llm-additional}

This appendix gives the experimental details behind Section~\ref{sec:exp}. Appendix~\ref{app:synthetic-sensitivity} studies how the decay rate of the low-fidelity mismatch bound affects TACC. Appendix~\ref{app:llm-task-protocol} specifies the LLM-as-a-judge task, prompt-policy arms, weak-judge model, and baseline conventions. Appendix~\ref{app:llm-residual-full} reports the residual-mismatch regime for multiple high-fidelity costs. Appendix~\ref{app:llm-pure-decay} studies a vanishing-mismatch regime in which the low--high gap eventually disappears. Appendix~\ref{app:llm-empirical-gap} reports empirical weak-judge accuracy and low--high gap measurements from trained checkpoints, and Appendix~\ref{app:llm-empirical-bandit} uses those checkpoints in a bandit experiment.

\subsection{Sensitivity to the Low-Fidelity Improvement Rate}
\label{app:synthetic-sensitivity}

Using Bandit Set A, we vary the convergence rate $r$ in the mismatch bound $B_k(n)=O(n^{-r})$. A larger $r$ means that the low-fidelity source improves faster. Figure~\ref{fig:synthetic-sensitive-r} shows that TACC benefits from faster decay: the low-fidelity interval tightens more quickly, so more suboptimal arms can be certified through cheap continuation rather than high-fidelity escalation. This behavior is consistent with the mechanism analyzed in Section~\ref{sec:theory}.

\begin{figure}[t]
    \centering
    \includegraphics[width=0.78\linewidth]{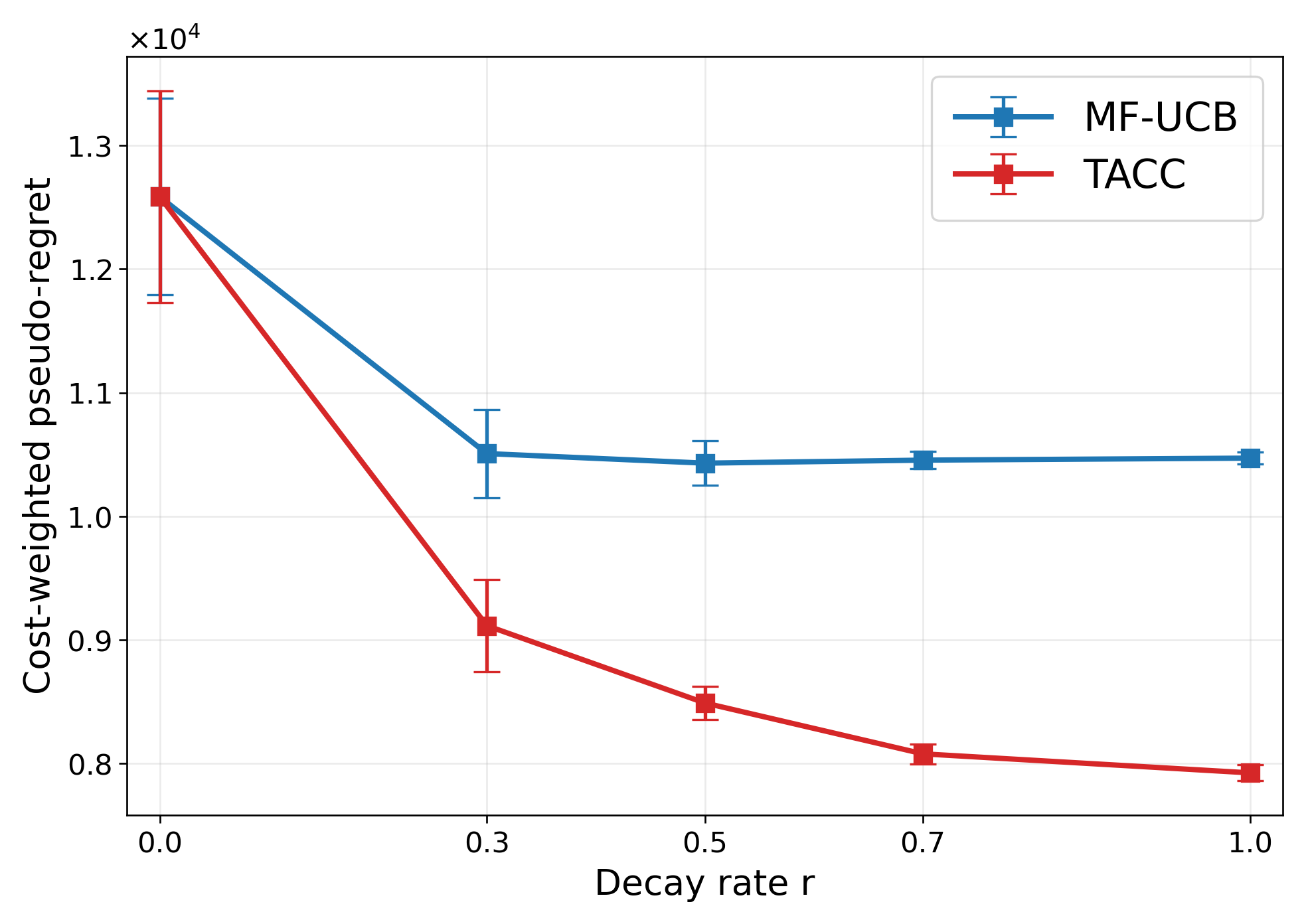}
    \caption{Sensitivity of TACC to the decay rate $r$ in the low-fidelity mismatch bound $B_k(n)=O(n^{-r})$. Larger $r$ corresponds to faster low-fidelity improvement.}
    \label{fig:synthetic-sensitive-r}
\end{figure}

\subsection{LLM-as-a-Judge Experimental Protocol}
\label{app:llm-task-protocol}

The LLM-as-a-judge experiment treats a bandit arm as a prompt-policy rather than as an individual NLI example. A prompt-policy maps an NLI input to a candidate label. The high-fidelity target is the expected verifier correctness of that policy, while the low-fidelity source is a cheaper weak-judge feedback process associated with the same policy.

\paragraph{Task and verifier.}
Each evaluation item consists of a premise $p$, a hypothesis $h$, and a verifier label $y\in\{\textsc{entailment},\textsc{neutral},\textsc{contradiction}\}$. A prompt-policy arm $k$ takes $(p,h)$ as input and returns a predicted label $\hat{y}_k(p,h)$. The high-fidelity verifier reward is
\begin{equation}
Y_k^{(H)}
=
\mathbf{1}\{\hat{y}_k(p,h)=y\},
\qquad
\mu_k^{(H)}
=
\mathbb{E}\!\left[Y_k^{(H)}\right],
\label{eq:app-llm-high-verifier}
\end{equation}
where the expectation is over the shared evaluation pool. In the bandit simulation, a high-fidelity query samples an evaluation item and observes verifier correctness for the selected prompt-policy. Regret is always computed with respect to the high-fidelity means $\mu_k^{(H)}$.

\paragraph{Prompt-policy arms.}
The prompt-policy pool includes \texttt{direct\_label}, \texttt{cot\_reasoning}, \texttt{role\_nli\_expert}, \texttt{self\_check}, \texttt{contrastive\_check}, \texttt{lexical\_overlap}, \texttt{hypothesis\_only}, and \texttt{neutral\_prior}. The main residual-mismatch experiment uses the four-arm subset \texttt{lexical\_overlap}, \texttt{hypothesis\_only}, \texttt{neutral\_prior}, and \texttt{contrastive\_check}; the checkpoint-based bandit experiment uses the five-arm subset specified in Appendix~\ref{app:llm-empirical-bandit}. Table~\ref{tab:app-llm-controlled-arms} summarizes the main four arms.

\begin{table}[t]
\centering
\caption{Controlled prompt-policy arms used in the main LLM-as-a-Judge experiment. All policies receive the same NLI item and return one label in $\{\textsc{entailment},\textsc{neutral},\textsc{contradiction}\}$.}
\label{tab:app-llm-controlled-arms}
\small
\setlength{\tabcolsep}{4pt}
{\renewcommand{\arraystretch}{1.08}
\begin{tabularx}{\linewidth}{>{\raggedright\arraybackslash\ttfamily}p{0.24\linewidth}X}
\toprule
Arm & Prompt-policy description \\
\midrule
lexical\_overlap
& Uses lexical overlap and containment between the premise and hypothesis; predicts entailment under high overlap or containment, contradiction under explicit negation cues, and neutral otherwise. \\
hypothesis\_only
& Ignores the premise and predicts from the hypothesis alone; this represents a spurious-prior policy induced by hypothesis-side artifacts. \\
neutral\_prior
& Favors the neutral label; this is a conservative weak distractor arm. \\
contrastive\_check
& Compares contradiction evidence and entailment evidence, returns the label with stronger support, and defaults to neutral when neither side is decisive. \\
\bottomrule
\end{tabularx}}
\end{table}

\begin{remark}
We use a four-arm benchmark in the main text to keep the main effect identifiable: the arms are fixed across all methods, so differences in regret come from fidelity allocation rather than from scaling to a large arm set. The inclusion of \texttt{hypothesis\_only} and lexical or prior-based policies is natural for NLI, where hypothesis-only baselines and annotation artifacts are known to create spurious predictive signals \citep{poliak2018hypothesis,gururangan2018annotation}.
\end{remark}

Other prompt-policies are used in the checkpoint-based diagnostic. \texttt{direct\_label} uses number mismatch, negation conflict, overlap, and containment; \texttt{cot\_reasoning} checks contradiction cues before entailment cues and then applies overlap thresholds; \texttt{role\_nli\_expert} is a more conservative expert-style rule; and \texttt{self\_check} takes a majority vote over several rule-based policies, preferring neutral on ties.

\paragraph{Controlled weak-judge model.}
For arm $k$, after its $n$-th low-fidelity query,
\begin{equation}
\mu_k^{(L)}(n)
=
\mu_k^{(H)}
+
b_k
+
a_k(n+n_0)^{-r}.
\label{eq:app-llm-controlled-weak-judge}
\end{equation}
Here $n$ is the arm-specific low-fidelity count. The transient term $a_k(n+n_0)^{-r}$ models improvement with repeated low-fidelity use. The residual term $b_k$ models persistent weak-judge bias relative to the verifier. In each seed, the low-fidelity process is fixed before running any bandit algorithm and is shared by all methods. The scale parameter $\zeta$ controls the initial low--high mismatch; in the residual-mismatch experiments, we set $b_k \equiv 0.05$ to enforce a persistent residual gap. When bounded observations are required, the resulting means are clipped to $[0,1]$ before sampling rewards.

The residual term $b_k$ is included because weak judges need not converge exactly to the target evaluator. MT-Bench and Chatbot Arena document judge biases such as position, verbosity, self-enhancement, and limited reasoning ability \citep{zheng2023judging}. Follow-up empirical studies also find that fine-tuned judges may lag stronger judges in generalizability, fairness, and adaptability \citep{huang2025empiricaljudge}. These findings motivate a model in which calibration improves the weak judge but need not eliminate the low--high gap.

The power-law transient in \eqref{eq:app-llm-controlled-weak-judge} is a modeling choice rather than an empirical claim about all judges. It gives a simple way to interpolate between an initially biased weak judge and a better but still imperfect judge.

\paragraph{Checkpoint weak judges.}
The empirical-gap experiments replace \eqref{eq:app-llm-controlled-weak-judge} with trained weak-judge checkpoints. The weak judge is \texttt{cross-encoder/nli-deberta-v3-small}. We use cumulative continued fine-tuning at checkpoint sizes $q\in\{32,64,128,256,512,1024,2048,4096,8192\}$, with maximum length $512$, $2$ stage epochs, learning rate $2\times 10^{-6}$, batch size $8$, evaluation batch size $32$, and seed $0$. Policy names are excluded from the weak-judge input to avoid policy-name leakage. The weak judge sees the task, premise, hypothesis, and candidate answer, and predicts whether the candidate answer matches the gold NLI relation.

For each checkpoint size $q$ and arm $k$, we estimate an empirical low-fidelity mean $\widehat{\mu}_{k,q}^{(L)}$ from weak-judge scores on the shared evaluation pool, and compare it with the verifier mean $\mu_k^{(H)}$. This gives a finite checkpoint-based low--high gap profile.

\paragraph{Baselines.}
All algorithms are evaluated on the same high-fidelity means, low-fidelity process, budget sequence, and random seeds when paired comparisons are reported. \textsc{UCB} uses only verifier feedback. \textsc{MF-UCB} uses a fixed multi-fidelity switching rule. \textsc{DNC} is the direct no-continuation ablation of TACC: it uses the same dynamic mismatch bound based confidence bounds and the same threshold $\gamma$, but after the low-fidelity radius crosses $\gamma$, it switches to high fidelity without checking whether additional low-fidelity samples would further reduce $B_k(n)$. \textsc{Weak-Fixed} uses only the initial weak judge. \textsc{Low-Improving} always uses the improving low-fidelity source and never falls back to the verifier; it is included as a diagnostic baseline.

\subsection{Residual-Mismatch Regime}
\label{app:llm-residual-full}

The main text reports the representative residual-mismatch experiment with $\lambda^{(H)}=500$. Here we give the full results for $\lambda^{(H)}\in\{200,500,1000\}$. We use the four prompt-policy arms from Table~\ref{tab:app-llm-controlled-arms}, $\lambda^{(L)}=1$, $r=0.75$, $\zeta=0.4$, $b_k \equiv 0.05$, $\gamma=0.025$, $\eta=10^{-4}$, and $S_0=128$. Budgets range up to $\Lambda=128000$, and results are averaged over $200$ seeds.

Table~\ref{tab:app-llm-residual-summary} reports final-budget regret for all methods and all three high-fidelity costs. Figure~\ref{fig:app-llm-residual-curves} shows the corresponding six-method regret curves. TACC remains competitive across all three high-fidelity costs. The residual mismatch makes \textsc{Low-Improving} unreliable, while the improvement of the weak judge still leaves room for bounded continuation to reduce costly high-fidelity confirmations.

\begin{table*}[t]
\centering
\caption{
Residual-mismatch experiment at final budget $\Lambda=128000$. Values are mean
cost-weighted pseudo-regret $\pm$ standard error over $200$ seeds. Lower is
better; the best method for each high-fidelity cost is bolded.
}
\label{tab:app-llm-residual-summary}
\small
\setlength{\tabcolsep}{4pt}
\begin{tabular}{c c c c c c c}
\toprule
$\lambda^{(H)}$ & TACC & DNC & MF-UCB & UCB & Weak-Fixed & Low-Improving \\
\midrule
$200$
& $\mathbf{3277.7 \pm 243.2}$
& $3894.0 \pm 266.6$
& $4317.0 \pm 276.8$
& $4264.1 \pm 289.1$
& $7331.9 \pm 119.3$
& $6827.5 \pm 85.7$ \\
$500$
& $\mathbf{4023.4 \pm 247.3}$
& $5201.0 \pm 289.7$
& $5359.1 \pm 281.8$
& $5083.2 \pm 286.8$
& $7501.6 \pm 127.7$
& $6808.4 \pm 83.4$ \\
$1000$
& $\mathbf{4688.6 \pm 253.6}$
& $5532.7 \pm 277.0$
& $5334.4 \pm 251.1$
& $5190.7 \pm 253.9$
& $7058.4 \pm 101.2$
& $6668.6 \pm 50.9$ \\
\bottomrule
\end{tabular}
\end{table*}

For the representative cost $\lambda^{(H)}=500$, Table~\ref{tab:app-llm-residual-paired} gives paired final-budget comparisons. The TACC-minus-baseline intervals are below zero for all primary baselines, including \textsc{DNC}; this isolates the contribution of the continuation block in the residual-mismatch setting.

\begin{table}[t]
\centering
\caption{Paired final-budget comparisons for the residual-mismatch experiment with $\lambda^{(H)}=500$. The reported value is TACC minus baseline; negative values favor TACC. Confidence intervals are paired $95\%$ intervals over $200$ common seeds.}
\label{tab:app-llm-residual-paired}
\small
\setlength{\tabcolsep}{5pt}
\begin{tabular}{lccc}
\toprule
Baseline & Mean diff. & $95\%$ CI & Favors TACC? \\
\midrule
DNC & $-561.1$ & $[-1103.6,\,-18.6]$ & Yes \\
MF-UCB & $-631.4$ & $[-1208.5,\,-54.4]$ & Yes \\
UCB & $-619.4$ & $[-1227.2,\,-11.6]$ & Yes \\
Weak-Fixed & $-2865.0$ & $[-3318.3,\,-2411.6]$ & Yes \\
\bottomrule
\end{tabular}
\end{table}

Figure~\ref{fig:app-llm-residual-curves} gives the full regret curves for all six methods and all three high-fidelity costs.

\begin{figure*}[t]
    \centering
    \begin{subfigure}[t]{0.48\textwidth}
        \centering
        \includegraphics[width=\linewidth]{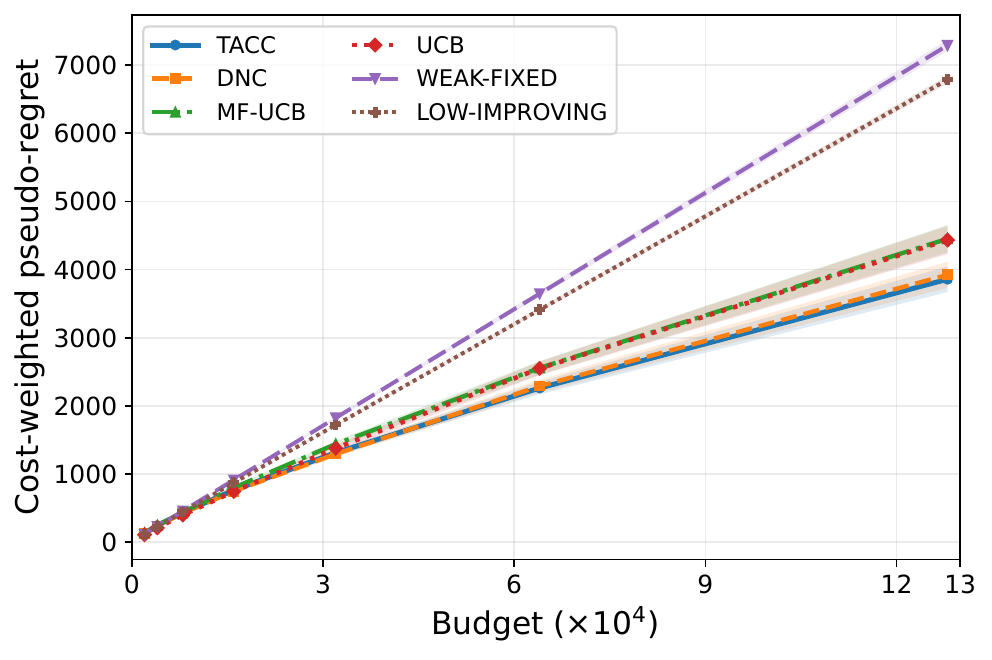}
        \caption{$\lambda^{(H)}=200$}
    \end{subfigure}
    \hfill
    \begin{subfigure}[t]{0.48\textwidth}
        \centering
        \includegraphics[width=\linewidth]{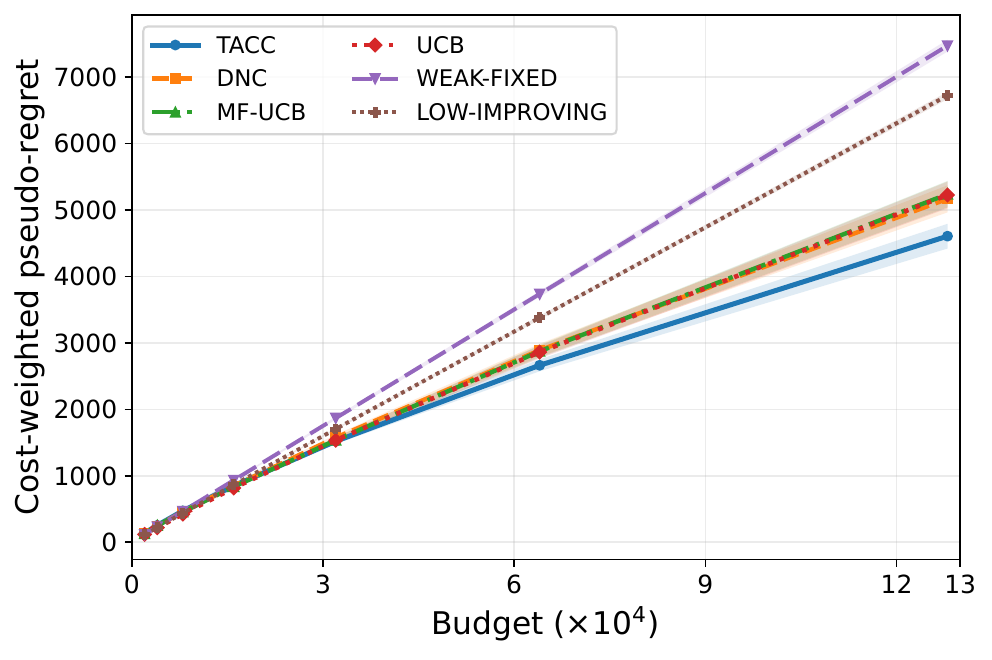}
        \caption{$\lambda^{(H)}=500$}
    \end{subfigure}

    \vspace{0.8em}

    \begin{subfigure}[t]{0.48\textwidth}
        \centering
        \includegraphics[width=\linewidth]{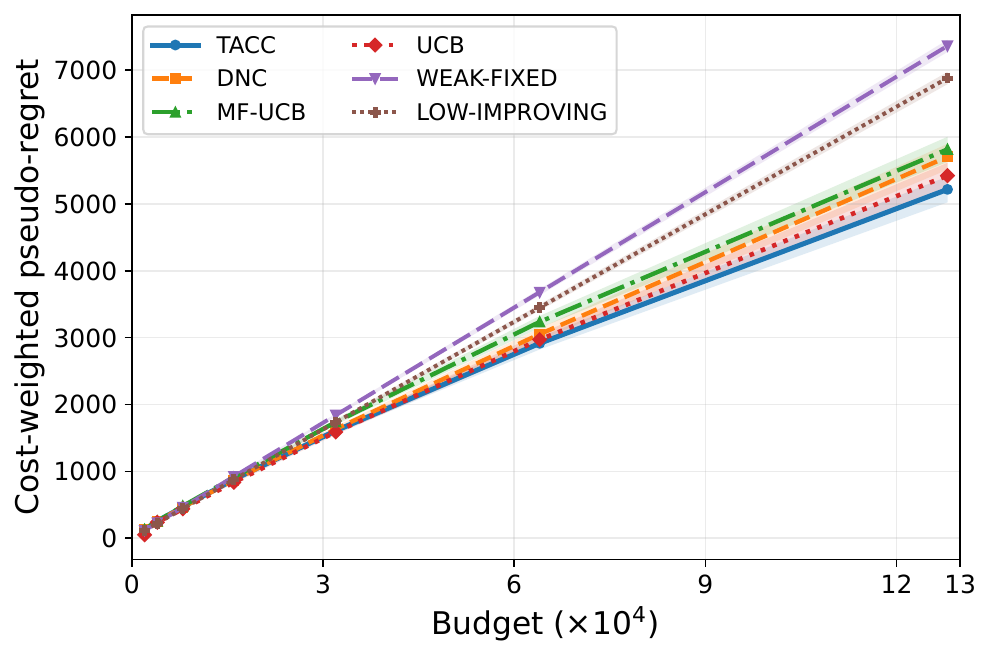}
        \caption{$\lambda^{(H)}=1000$}
    \end{subfigure}
    \caption{Residual-mismatch regret curves for $\lambda^{(H)}\in\{200,500,1000\}$. The weak judge improves with use but retains persistent mismatch from the verifier.}
    \label{fig:app-llm-residual-curves}
\end{figure*}

Figure~\ref{fig:app-llm-residual-continuation} reports the continuation counts for the representative $\lambda^{(H)}=500$ setting. The continuation budget remains bounded, consistent with the design of TACC as a post-threshold correction rather than a replacement for high-fidelity verification.

\begin{figure}[t]
    \centering
    \includegraphics[width=0.88\linewidth]{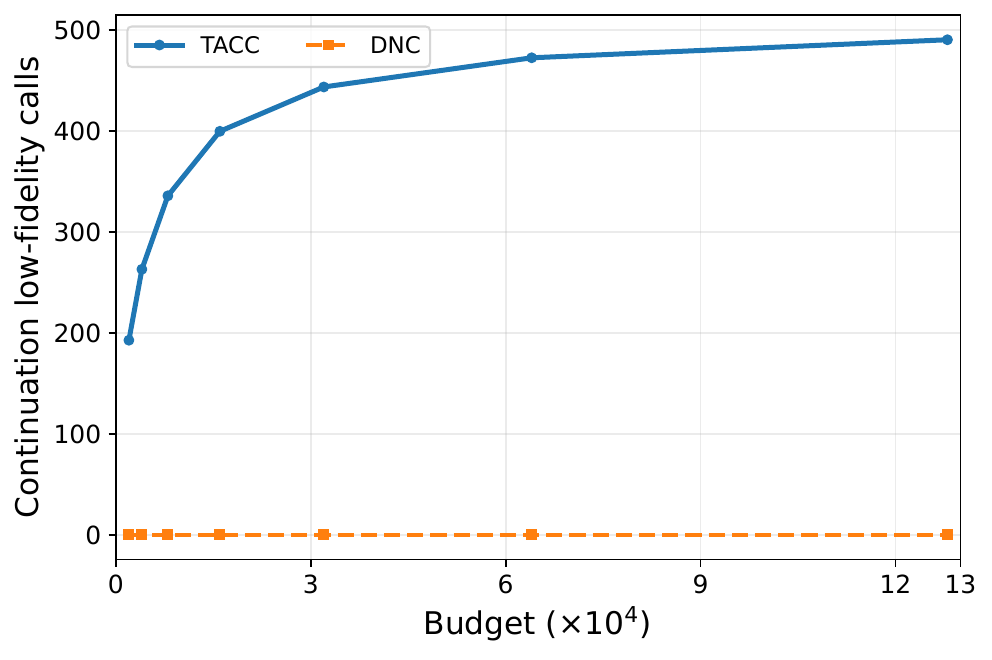}
    \caption{Continuation calls in the residual-mismatch experiment with $\lambda^{(H)}=500$.}
    \label{fig:app-llm-residual-continuation}
\end{figure}

\subsection{Vanishing-Mismatch Regime: \texorpdfstring{$B_k(n)\to 0$}{Bk(n) to 0}}
\label{app:llm-pure-decay}

We next remove the residual term and consider
\begin{equation}
    \mu_k^{(L)}(n)
    =
    \mu_k^{(H)}
    +
    a_k(n+n_0)^{-r}.
    \label{eq:app-llm-pure-decay-model}
\end{equation}
This is the regime where the low--high mismatch shrinks toward zero. In this case, low fidelity eventually becomes aligned with the high-fidelity target, so \textsc{Low-Improving} can become competitive at large budgets.

Table~\ref{tab:app-llm-pure-decay-summary} reports final-budget regret for the main methods and \textsc{Weak-Fixed}, and Figure~\ref{fig:app-llm-pure-decay-curves} shows the corresponding six-method regret curves. TACC has small final regret across all high-fidelity costs. We interpret this experiment as a theory-aligned diagnostic rather than the hardest practical LLM-judge regime: once the cheap source effectively converges to the target, a low-only improving strategy can also become competitive.

\begin{table*}[t]
\centering
\caption{
Vanishing-mismatch experiment at final budget $\Lambda=128000$. Values are mean
cost-weighted pseudo-regret $\pm$ standard error. Lower is better; the best
reported method for each high-fidelity cost is bolded.
}
\label{tab:app-llm-pure-decay-summary}
\small
\setlength{\tabcolsep}{4pt}
\begin{tabular}{c c c c c c}
\toprule
$\lambda^{(H)}$ & TACC & DNC & MF-UCB & UCB & Weak-Fixed \\
\midrule
$200$
& $\mathbf{75.7 \pm 0.5}$
& $677.5 \pm 28.6$
& $1503.1 \pm 37.1$
& $1506.1 \pm 39.4$
& $6503.0 \pm 0.0$ \\
$500$
& $\mathbf{143.0 \pm 0.3}$
& $657.7 \pm 113.4$
& $2687.8 \pm 74.0$
& $2711.1 \pm 52.0$
& $6503.0 \pm 0.0$ \\
$1000$
& $\mathbf{261.8 \pm 0.4}$
& $598.7 \pm 187.9$
& $3719.0 \pm 94.4$
& $3790.0 \pm 64.1$
& $6503.0 \pm 0.0$ \\
\bottomrule
\end{tabular}
\end{table*}


\begin{figure*}[t]
    \centering
    \begin{subfigure}[t]{0.48\textwidth}
        \centering
        \includegraphics[width=\linewidth]{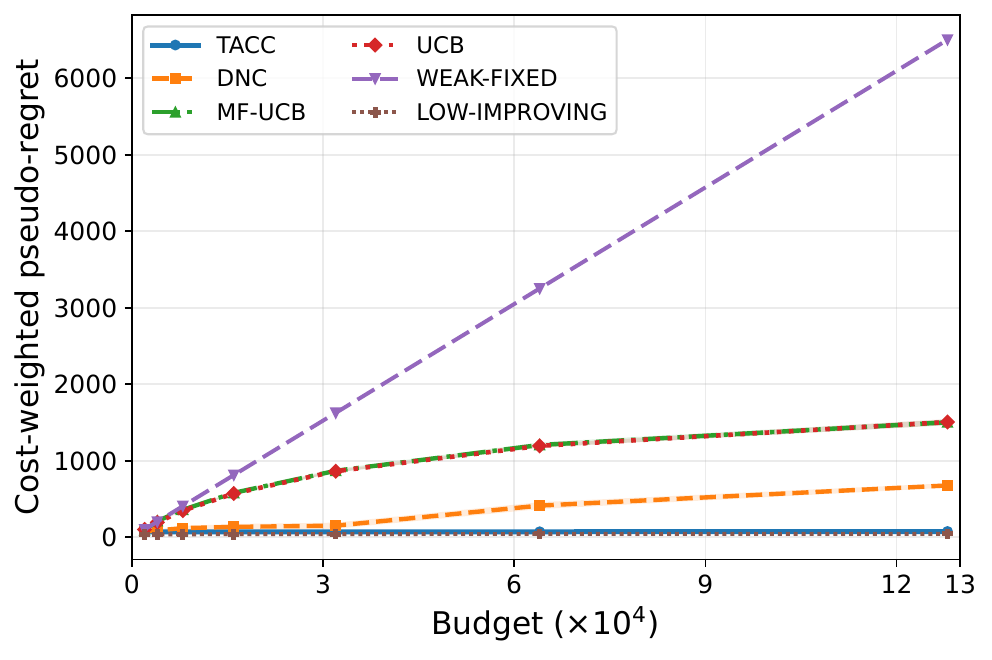}
        \caption{$\lambda^{(H)}=200$}
    \end{subfigure}
    \hfill
    \begin{subfigure}[t]{0.48\textwidth}
        \centering
        \includegraphics[width=\linewidth]{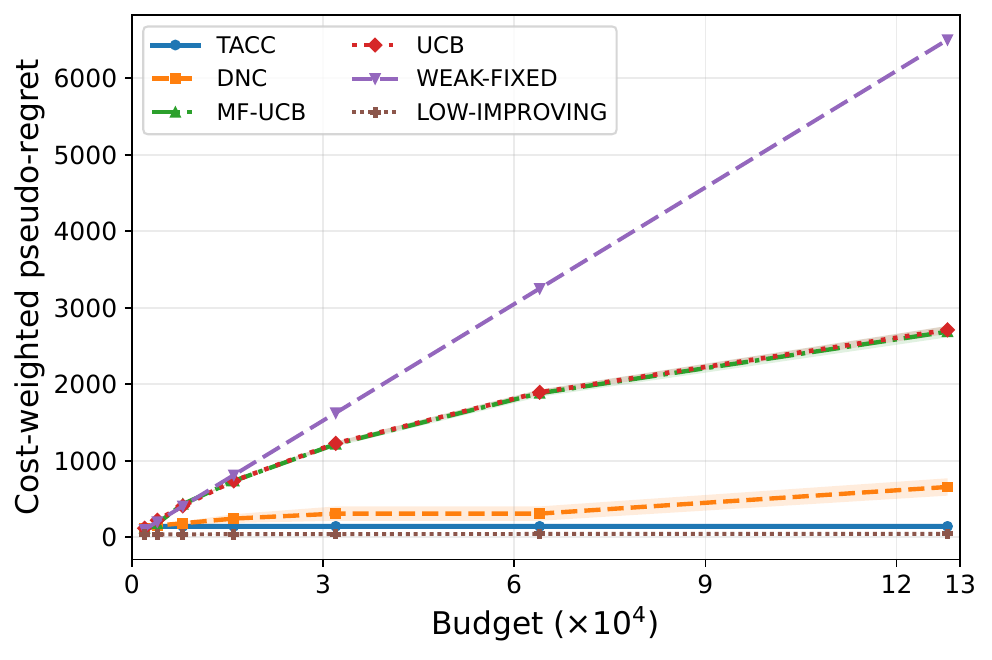}
        \caption{$\lambda^{(H)}=500$}
    \end{subfigure}

    \vspace{0.8em}

    \begin{subfigure}[t]{0.48\textwidth}
        \centering
        \includegraphics[width=\linewidth]{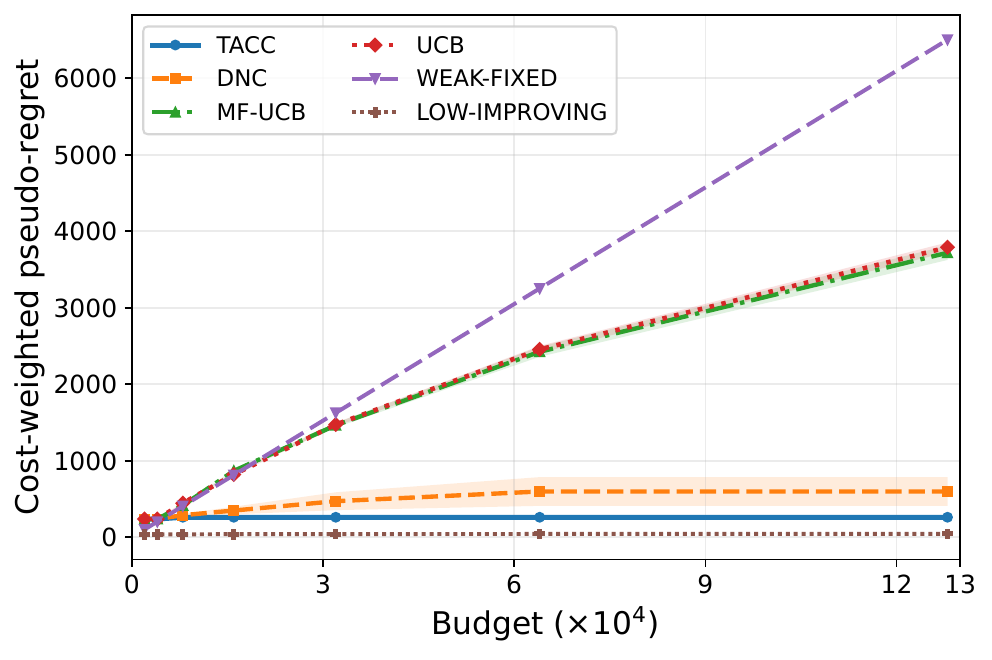}
        \caption{$\lambda^{(H)}=1000$}
    \end{subfigure}
    \caption{Vanishing-mismatch regret curves for $\lambda^{(H)}\in\{200,500,1000\}$. The low-fidelity source eventually aligns with the high-fidelity target.}
    \label{fig:app-llm-pure-decay-curves}
\end{figure*}

\subsection{Empirical Weak-Judge Gap}
\label{app:llm-empirical-gap}

The residual-mismatch and vanishing-mismatch experiments control the low-fidelity process analytically. We also evaluate a checkpoint-based construction. Rather than defining low fidelity by algebraically perturbing the high-fidelity mean, we train weak judges at increasing data scales and estimate the low--high gap from their policy-level behavior.

Figure~\ref{fig:app-weak-judge-accuracy} and Table~\ref{tab:app-weak-judge-accuracy} report weak-judge prediction accuracy across checkpoints. Accuracy increases from $0.4881$ at $q=32$ to $0.5669$ at $q=512$, $0.6344$ at $q=1024$, $0.7040$ at $q=2048$, and $0.7619$ at $q=8192$. The mean loss also decreases from $0.7892$ at $q=32$ to $0.6922$ at $q=512$, $0.6096$ at $q=2048$, and $0.5291$ at $q=8192$.

\begin{figure}[t]
    \centering
    \includegraphics[width=0.78\linewidth]{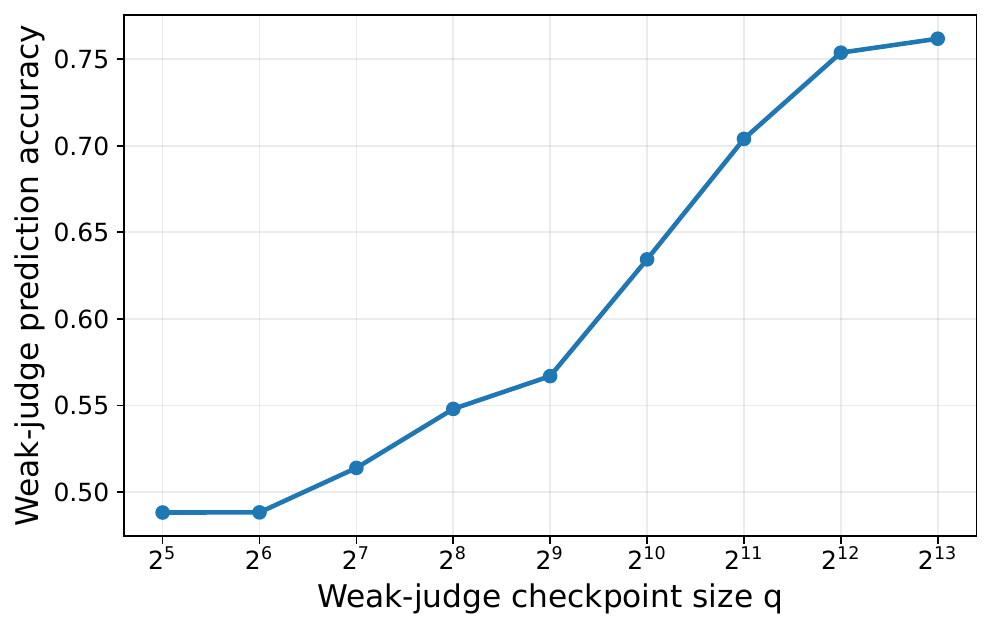}
    \caption{Weak-judge prediction accuracy as a function of checkpoint size $q$.}
    \label{fig:app-weak-judge-accuracy}
\end{figure}

\begin{table}[t]
\centering
\caption{Weak-judge prediction accuracy across checkpoint sizes.}
\label{tab:app-weak-judge-accuracy}
\small
\setlength{\tabcolsep}{6pt}
\begin{tabular}{cc}
\toprule
$q$ & Prediction accuracy \\
\midrule
$32$ & $0.4881$ \\
$64$ & $0.4882$ \\
$128$ & $0.5139$ \\
$256$ & $0.5480$ \\
$512$ & $0.5669$ \\
$1024$ & $0.6344$ \\
$2048$ & $0.7040$ \\
$4096$ & $0.7538$ \\
$8192$ & $0.7619$ \\
\bottomrule
\end{tabular}
\end{table}

For each prompt-policy arm and weak-judge scale $q$, we compute an empirical low-fidelity policy mean and compare it with the verifier mean. Figure~\ref{fig:app-low-high-gap} and Table~\ref{tab:app-llm-empirical-gap-diagnostics} report the resulting policy-level low--high gap. The mean absolute gap decreases from $0.1755$ at $q=32$ to $0.0452$ at $q=512$ and $0.0318$ at $q=1024$, with later checkpoints fluctuating at a lower level. These measurements support the modeling premise that cheap feedback can be initially unreliable but becomes more aligned with the high-fidelity verifier after additional calibration.

\begin{figure}[t]
    \centering
    \includegraphics[width=0.78\linewidth]{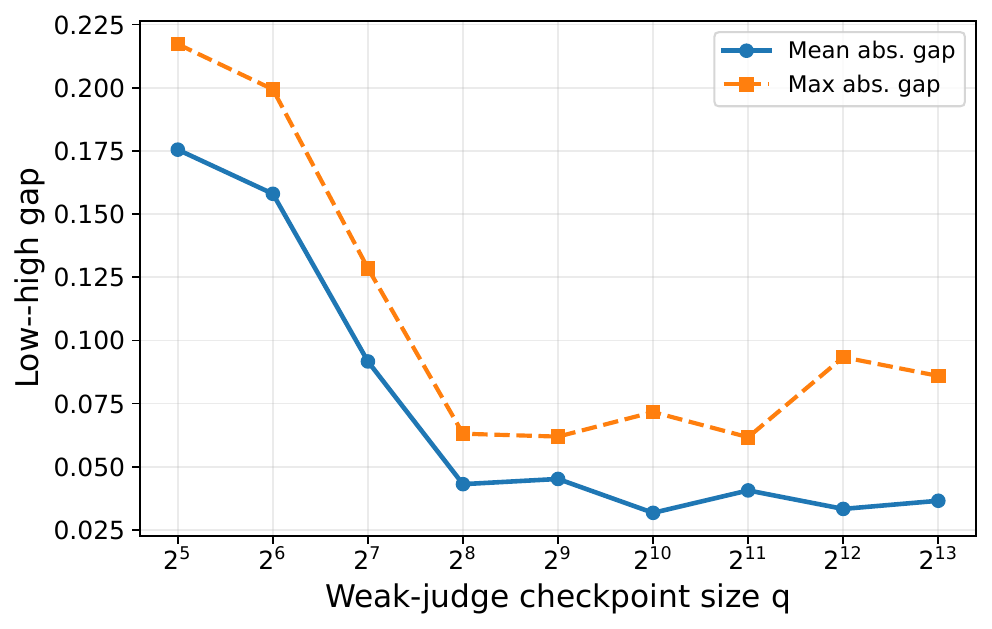}
    \caption{Policy-level low--high gap as a function of weak-judge checkpoint size $q$.}
    \label{fig:app-low-high-gap}
\end{figure}

\begin{table}[t]
\centering
\caption{Empirical weak-judge gap measurements on the evaluation split.}
\label{tab:app-llm-empirical-gap-diagnostics}
\small
\setlength{\tabcolsep}{5pt}
\begin{tabular}{ccc}
\toprule
$q$ & Mean abs. gap & Max abs. gap \\
\midrule
$32$ & $0.1755$ & $0.2172$ \\
$64$ & $0.1580$ & $0.1993$ \\
$128$ & $0.0917$ & $0.1286$ \\
$256$ & $0.0431$ & $0.0631$ \\
$512$ & $0.0452$ & $0.0620$ \\
$1024$ & $0.0318$ & $0.0717$ \\
$2048$ & $0.0407$ & $0.0617$ \\
$4096$ & $0.0333$ & $0.0934$ \\
$8192$ & $0.0366$ & $0.0860$ \\
\bottomrule
\end{tabular}
\end{table}

\subsection{Checkpoint-Based Bandit Experiment}
\label{app:llm-empirical-bandit}

Finally, we use the empirical weak-judge checkpoints from Appendix~\ref{app:llm-empirical-gap} inside the bandit simulation. This experiment is checkpoint-based rather than rate-aligned: the weak-judge process is finite, capped at $q_{\max}=2048$, and not exactly monotone. The regret curve should therefore be read as a finite-budget diagnostic, not as an asymptotic-rate demonstration.

We report a five-arm setting with \texttt{lexical\_overlap}, \texttt{direct\_label}, \texttt{self\_check}, \texttt{cot\_reasoning}, and \texttt{hypothesis\_only}. The high-fidelity means are $0.5568$, $0.5376$, $0.5376$, $0.5364$, and $0.5060$, respectively. We use $\lambda^{(H)}=500$, $\lambda^{(L)}=1$, $\gamma=0.03$, $\eta=10^{-4}$, $S_0=128$, $q_{\max}=2048$, budgets up to $\Lambda=256000$, and $120$ seeds.

Table~\ref{tab:app-llm-empirical-checkpoint} gives paired final-budget comparisons, and Figure~\ref{fig:app-llm-empirical-checkpoint} shows the corresponding regret curves. TACC significantly improves over \textsc{UCB} and \textsc{Weak-Fixed}. It also has lower mean regret than \textsc{MF-UCB} and \textsc{DNC}, although the paired confidence intervals overlap zero. \textsc{Low-Improving} is better at the final budget, which is expected when later weak-judge checkpoints are substantially more aligned with the verifier. We therefore interpret this experiment conservatively: it supports the practical relevance of improving weak-judge feedback, while the synthetic and residual-mismatch experiments provide cleaner evidence for the continuation mechanism.

\begin{table}[t]
\centering
\caption{Checkpoint-based bandit experiment at final budget $\Lambda=256000$. The reported value is TACC minus baseline; negative values favor TACC. Confidence intervals are paired $95\%$ intervals over $120$ common seeds.}
\label{tab:app-llm-empirical-checkpoint}
\small
\setlength{\tabcolsep}{4pt}
\begin{tabular}{lccc}
\toprule
Baseline & Mean diff. & $95\%$ CI & Favors TACC? \\
\midrule
DNC & $-737.5$ & $[-1606.2,\,131.1]$ & No \\
MF-UCB & $-427.7$ & $[-1125.4,\,270.1]$ & No \\
UCB & $-2356.3$ & $[-2831.0,\,-1881.5]$ & Yes \\
Weak-Fixed & $-2803.1$ & $[-3273.9,\,-2332.3]$ & Yes \\
Low-Improving & $1989.8$ & $[1521.2,\,2458.3]$ & No \\
\bottomrule
\end{tabular}
\end{table}

\begin{figure}[t]
    \centering
    \includegraphics[width=0.82\linewidth]{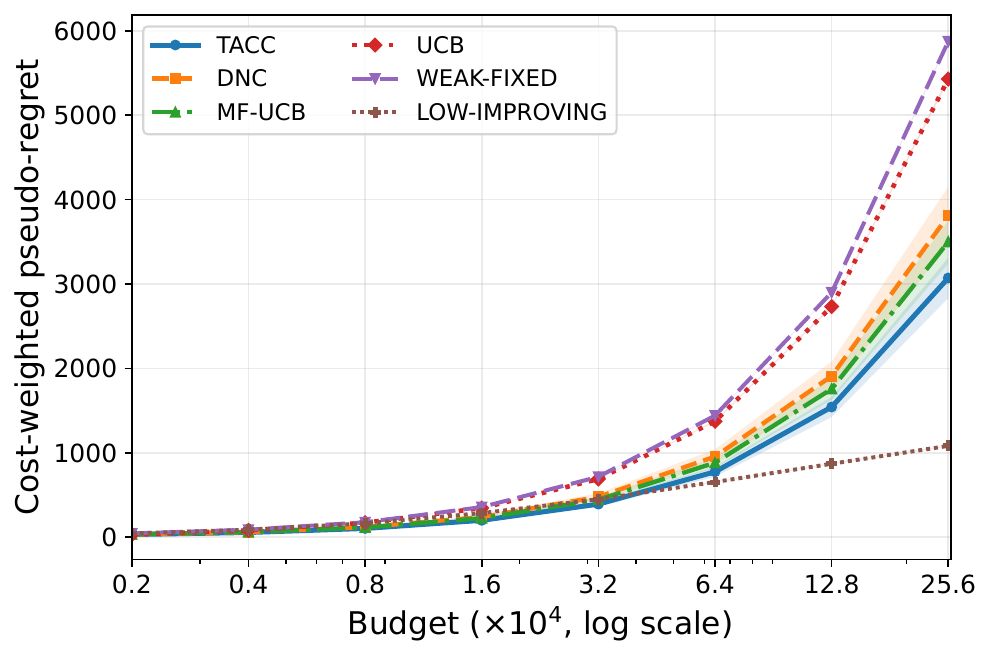}
    \caption{Checkpoint-based bandit experiment with $\lambda^{(H)}=500$ and $q_{\max}=2048$. The low-fidelity process is estimated from trained weak-judge checkpoints rather than imposed algebraically from the high-fidelity means.}
    \label{fig:app-llm-empirical-checkpoint}
\end{figure}

Figure~\ref{fig:app-llm-empirical-extra} separates the aggregate regret curve into normalized regret, high-fidelity calls, and continuation calls. These diagnostics help identify whether differences come from verifier usage, continuation behavior, or the cumulative scale of the budget.

\begin{figure*}[t]
    \centering
    \begin{subfigure}[t]{0.48\textwidth}
        \centering
        \includegraphics[width=\linewidth]{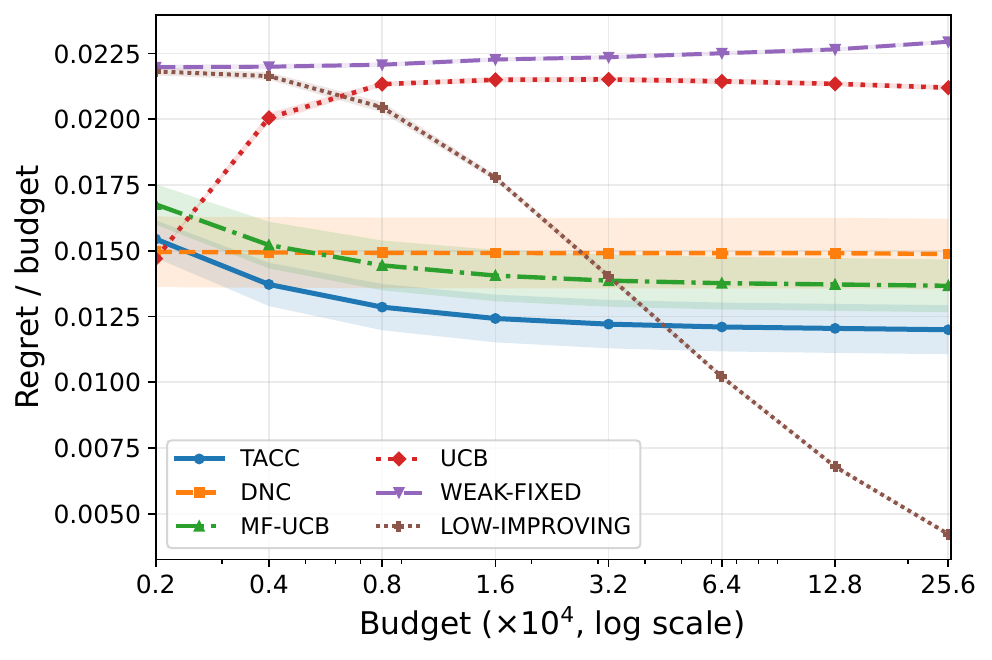}
        \caption{Normalized regret}
    \end{subfigure}
    \hfill
    \begin{subfigure}[t]{0.48\textwidth}
        \centering
        \includegraphics[width=\linewidth]{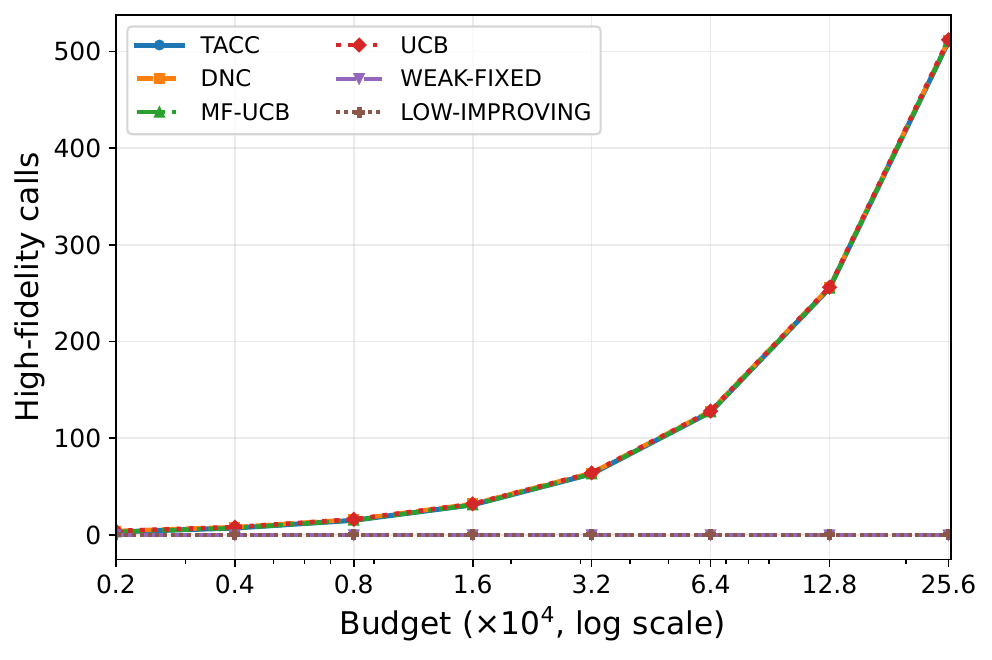}
        \caption{High-fidelity calls}
    \end{subfigure}

    \vspace{0.8em}

    \begin{subfigure}[t]{0.48\textwidth}
        \centering
        \includegraphics[width=\linewidth]{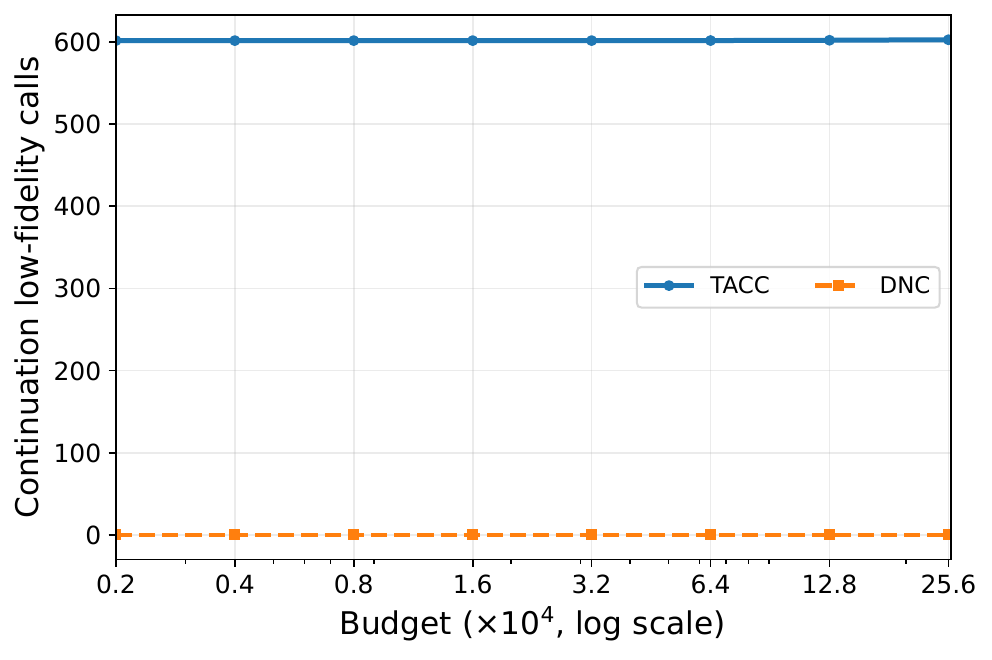}
        \caption{Continuation calls}
    \end{subfigure}
    \caption{Additional diagnostics for the checkpoint-based bandit experiment.}
    \label{fig:app-llm-empirical-extra}
\end{figure*}

The LLM-as-a-Judge experiments should not be read as visual estimates of an asymptotic regret exponent. The synthetic and vanishing-mismatch experiments test the shrinking-mismatch bound mechanism most directly. The residual-mismatch and checkpoint-based LLM experiments instead test finite-budget robustness under weak-judge mismatch, checkpoint caps, close top arms, and nonmonotone trained checkpoints.

\clearpage


\end{document}